\documentclass{article}

\usepackage[final]{neurips_2019}

\usepackage[utf8]{inputenc} 
\usepackage[T1]{fontenc}    
\usepackage{url}            
\usepackage{booktabs}       
\usepackage{amsfonts}       
\usepackage{nicefrac}       
\usepackage{microtype}      

\usepackage{natbib}
\usepackage{mathtools}
\usepackage{cite}
\usepackage[utf8]{inputenc}

\usepackage{graphicx}
\usepackage{subcaption}
\usepackage{pstricks}
\usepackage{color}
\usepackage{bbm}
\usepackage{tcolorbox}
\usepackage{tikz}
\usepackage{pgfplots}
\usepackage{wrapfig}
\usepackage{lipsum}  
\usetikzlibrary{positioning}
\usepackage{tcolorbox}
\usepackage{amssymb}
\usepackage{amsmath,amsfonts,amssymb,amsthm}

\usepackage{mathtools}
\usepackage{commath}
\usepackage{relsize}
\usepackage{bbm}
\usepackage{bm}
\usepackage[font={small}]{caption}
\usepackage{comment,color,soul}
\usepackage{enumerate} 
\usetikzlibrary{arrows,automata}
\usepackage{amsfonts}
\usepackage{enumitem}
\usepackage{mathtools}
\usepackage{url}
\usepackage{algorithm}
\usepackage{algpseudocode}
\usepackage{lipsum}
\usepackage[thinlines]{easytable}
\makeatletter
\newcommand{\algmargin}{\the\ALG@thistlm}
\makeatother
\newlength{\whilewidth}
\algdef{SE}[parWHILE]{parWhile}{EndparWhile}[1]
  {\parbox[t]{\dimexpr\linewidth-\algmargin}{%
     \hangindent\whilewidth\strut\algorithmicwhile\ #1\ \algorithmicdo\strut}}{\algorithmicend\ \algorithmicwhile}%
\algnewcommand{\parState}[1]{\State%
  \parbox[t]{\dimexpr\linewidth-\algmargin}{\strut #1\strut}}
  
\usepackage{mysymbol}
\makeatletter
\def\thm@space@setup{\thm@preskip=2pt
        \thm@postskip=2pt \itshape}
\makeatother

\newtheoremstyle{newstyle}
{} 
{} 
{\mdseries} 
{} 
{\bfseries} 
{.} 
{ } 
{} 

\theoremstyle{definition}
\newtheorem{assumption}{Assumption}

\newtheorem{theorem}{Theorem}
\newtheorem{remark}{Remark}

\newtheorem{lemma}{Lemma}

\newcommand{\gr}{\nabla} 
\newcommand{\al}{\alpha} 
\newcommand{\parf}{\partial f} 
\newcommand{\tparf}{\widetilde{\partial} f}
\newcommand{\bx}{\mathbf{x}}
\newcommand{\bW}{\mathbf{W}}
\newcommand{\bI}{\mathbf{I}}

\newcommand{\bz}{\mathbf{z}}
\newcommand{\bu}{\mathbf{u}}
\newcommand{\bE}{\mathbb{E}}

\newcommand{\cF}{\mathcal{F}}
\newcommand{\cO}{\mathcal{O}}
\newcommand{\cD}{\mathcal{D}}

\newcommand{\bone}{\mathbf{1}}
\newcommand{\tildbx}{\widetilde{\mathbf{x}}}

\newcommand{\tNab}{\widetilde{\nabla}}

\usepackage[hidelinks]{hyperref}
\definecolor{darkred}{RGB}{150,0,0}
\definecolor{darkgreen}{RGB}{0,150,0}
\definecolor{darkblue}{RGB}{0,0,150}
\hypersetup{colorlinks=true, linkcolor=red, citecolor=blue, urlcolor=darkblue}

\title{Robust and Communication-Efficient \\ Collaborative Learning}

\author{Amirhossein Reisizadeh \\
  ECE Department \\
  University of California, Santa Barbara\\
  \texttt{reisizadeh@ucsb.edu}
  \And
  Hossein Taheri \\
  ECE Department\\
  University of California, Santa Barbara\\
  \texttt{hossein@ucsb.edu}\\
  \AND
  Aryan Mokhtari \\
  ECE Department\\
  The University of Texas at Austin \\
  \texttt{mokhtari@austin.utexas.edu}
     \And
  Hamed Hassani \\
  ESE Department \\
  University of Pennsylvania\\
  \texttt{hassani@seas.upenn.edu}\\
     \And
  Ramtin Pedarsani \\
  ECE Department\\
  University of California, Santa Barbara\\
  \texttt{ramtin@ece.ucsb.edu}\\
}

\begin{document}

\maketitle
\begin{abstract}

We consider a decentralized learning problem, where a set of computing nodes aim at solving a non-convex optimization problem collaboratively. It is well-known that decentralized optimization schemes face two major system bottlenecks: stragglers' delay and communication overhead. In this paper, we tackle these bottlenecks by proposing  a novel decentralized and gradient-based optimization algorithm named as \texttt{QuanTimed-DSGD}. Our algorithm stands on two main ideas: (i) we impose a \emph{deadline} on the local gradient computations of each node at each iteration of the algorithm, and (ii) the nodes exchange \emph{quantized} versions of their local models. The first idea robustifies to straggling nodes and the second alleviates communication efficiency. The key technical contribution of our work is to prove that with non-vanishing noises for quantization and stochastic gradients, the proposed method \emph{exactly} converges to the global optimal for convex loss functions, and finds a first-order stationary point in non-convex scenarios. Our numerical evaluations of the \texttt{QuanTimed-DSGD} on training benchmark datasets, MNIST and CIFAR-10, demonstrate speedups of up to $3 \times$ in run-time, compared  to state-of-the-art decentralized optimization methods.

\end{abstract}

\section{Introduction}\label{sec:intro}


Collaborative learning refers to the task of learning a common objective among multiple computing agents without any central node and by using on-device computation and local communication among neighboring agents. Such tasks have recently gained considerable attention in the context of machine learning and optimization as they are foundational to several computing paradigms such as scalability to larger datasets and systems, data locality, ownership and privacy. As such, collaborative learning naturally arises in various applications such as distributed deep learning \citep{lecun2015deep,dean2012large}, multi-agent robotics and path planning \citep{choi2010continuous,jha2016path}, distributed resource allocation in wireless networks \citep{ribeiro2010ergodic},
 to name a few. 


While collaborative learning has recently drawn significant attention due its decentralized implementation, it faces major challenges at the system level as well as algorithm design. The decentralized implementation of collaborative learning faces two major systems challenges: (i) significant slow-down due to straggling nodes, where a subset of nodes can be largely delayed in their local computation which slows down the wall-clock time convergence of the decentralized algorithm; (ii) large communication overhead due to the message passing algorithm as the dimension of the parameter vector increases, which can further slow down the algorithm's convergence time. Moreover, in the presence of these system bottlenecks, the efficacy of classical consensus optimization methods is not clear and needs to be revisited. 

In this work we consider the general data-parallel setting where the data is distributed across
different computing nodes, and develop decentralized optimization methods that do not rely on a central
coordinator but instead only require local computation and communication among neighboring nodes. 
As the main contribution of this paper, we propose a \textit{straggler-robust} and \textit{communication-efficient} algorithm for collaborative learning called 
\texttt{QuanTimed-DSGD}, which is a quantized and deadline-based decentralized stochastic gradient descent method.  We show that the proposed scheme provably improves upon on the convergence time of vanilla synchronous decentralized optimization methods. The key theoretical contribution of the paper is to develop the \emph{first} quantized decentralized non-convex optimization algorithm with provable and \emph{exact} convergence to a first-order optimal solution. 

There are two key ideas in our proposed algorithm. To provide robustness against stragglers, we impose a \emph{deadline} time $T_d$ for the computation of each node. In a synchronous implementation of the proposed algorithm, at every iteration all the nodes simultaneously start computing stochastic gradients by randomly picking data points from their local batches and evaluating the gradient function on the picked data point. By $T_d$, each node has computed a random number of stochastic gradients from which it aggregates and generates a stochastic gradient for its local objective. By doing so, each iteration takes a constant computation time as opposed to deadline-free methods in which each node has to wait for all their neighbours to complete their gradient computation tasks. 
To tackle the communication bottleneck in collaborative learning, we only allow the decentralized nodes to share with neighbours a \emph{quantized} version of their local models. Quantizing the exchanged models reduces the communication load which is critical for large and dense networks.

We analyze the convergence of the proposed \texttt{QuanTimed-DSGD} for strongly convex and non-convex loss functions and under standard assumptions for the network, quantizer and stochastic gradients. In the strongly convex case, we show that \texttt{QuanTimed-DSGD} \emph{exactly} finds the global optimal for \emph{every} node with a rate arbitrarily close to $\ccalO(1/\sqrt{T})$. 
In the non-convex setting,  \texttt{QuanTimed-DSGD} provably finds first-order optimal solutions as fast as $\ccalO(T^{-1/3})$. Moreover, the consensus error decays with the same rate which guarantees an exact convergence by choosing large enough $T$. Furthermore, we numerically evaluate \texttt{QuanTimed-DSGD} on benchmark datasets CIFAR-10 and MNIST, where it demonstrates speedups of up to $3 \times$ in the run-time compared to state-of-the-art baselines.

\noindent \textbf{Related Work.} 
Decentralized consensus optimization has been studied extensively. The most popular first-order choices for the convex setting are distributed gradient descent-type methods \citep{Nedic2009,Jakovetic2014-1,yuan2016convergence,qu2017accelerated}, augmented Lagrangian algorithms \citep{shi2015extra,shi2015proximal,mokhtari2016dsa}, distributed variants of the alternating direction method of multipliers (ADMM) \citep{Schizas2008-1,BoydEtalADMM11,Shi2014-ADMM,chang2015multi,mokhtari2016dqm}, dual averaging \citep{Duchi2012,cTsianosEtal12}, and several dual based strategies \citep{seaman2017optimal,scaman2018optimal,uribe2018dual}. Recently, there have been some works which study non-convex decentralized consensus optimization and establish convergence to a stationary point \citep{zeng2018nonconvex,Hong_Prox-PDA,hong2018gradient,sun2018distributed,scutari2017parallel,scutari2018distributed,jiang2017collaborative,lian2017can}.

The idea of improving communication-efficiency of distributed optimization procedures via message-compression schemes goes a few decades back \citep{tsitsiklis1987communication}, however, it has recently gained considerable attention due to the growing importance of distributed applications. In particular, efficient gradient-compression methods are provided in \citep{alistarh2017qsgd,seide20141,bernstein2018signsgd} and deployed in the distributed master-worker setting. In the decentralized setting, quantization methods were proposed in different convex optimization contexts with \emph{non-vanishing} errors \citep{yuksel2003quantization, rabbat2005quantized, Kashyap2006QuantizedC,el2016design,aysal2007distributed, nedic2008distributed}. The first \emph{exact} decentralized optimization method with quantized messages was given in \citep{reisizadeh2018quantized, zhang2018compressed}, and more recently, new techniques have been developed in this context for convex problems \citep{doan2018accelerating,koloskova2019decentralized, berahas2019nested,lee2018distributed,lee2018finite}.

The straggler problem has been widely observed in distributed computing clusters \citep{dean2013tail,ananthanarayanan2010reining}. A common approach to mitigate stragglers is to replicate the computing task of the slow nodes to other computing nodes \citep{ananthanarayanan2013effective,wang2014efficient}, but this is clearly not feasible in collaborative learning. Another line of work proposed using coding theoretic ideas for speeding up distributed machine learning \citep{lee2018speeding,tandon2016gradient,yu2017polynomial,reisizadeh2019coded,reisizadeh2019codedreduce}, but they work mostly for master-worker setup and particular computation types such as linear computations or full gradient aggregation. The closest work to ours is \citep{ferdinand2018anytime} that considers decentralized optimization for convex functions with deadline for local computations without considering communication bottlenecks and quantization as well as non-convex functions. Another line of work proposes asynchronous decentralized SGD, where the workers update their models based on the last iterates received by their neighbors \citep{HogWild!, lian2017asynchronous,lan2018asynchronous, peng2016convergence, wu2017decentralized,dutta2018slow}. While asynchronous methods are inherently robust to stragglers, they can suffer from slow convergence due to using stale models.

\section{Problem Setup}\label{sec:setup}

In this paper, we focus on a stochastic learning model in which we aim to solve the problem
\begin{equation}\label{population_risk}
  \min_{\bbx}  L(\bx) := \min_{\bbx} \mathbb{E}_{\theta \sim \ccalP} [\ell(\bx,\theta)],
\end{equation}
where $\ell:\mathbb{R}^p\times\mathbb{R}^q  \to\mathbb{R}$ is a stochastic loss function, $\bbx\in\mathbb{R}^p$ is our optimization variable, and  $\theta\in \mathbb{R}^q$ is a random variable with probability distribution $\ccalP$ and $L:\mathbb{R}^p \to\mathbb{R}$ is the expected loss function also called population risk. We assume that the underlying distribution $\ccalP$ of the random variable $\theta$ is unknown and we have access only to $N=mn$ realizations of it. 
Our goal is to solve the loss associated with  $N=mn$ realizations of the random variable $\theta$, which is also known as empirical risk minimization. To be more precise, we aim to solve the empirical risk minimization (ERM) problem
\begin{equation}\label{eq:ERM}
  \min_{\bbx}  L_N(\bbx) := \min_{\bbx} \frac{1}{N}\sum_{k=1}^N \ell(\bx,\theta_k),
\end{equation}
where $L_N$ is the empirical loss associated with the sample of random variables $\ccalD=\{\theta_1,\dots,\theta_N\}$. 

\textbf{Collaborative Learning Perspective}. Our goal is to solve the ERM problem in \eqref{eq:ERM} in a decentralized manner over $n$ nodes. This setting arises in a plethora of applications where either the total number of samples $N$ is massive and data cannot be stored or processed over a single node or the samples  are available in parts at different nodes and, due to privacy or communication constraints, exchanging raw data points is not possible among the nodes. Hence, we assume that each node $i$ has access to $m$ samples and its local objective is 
\begin{align}\label{local_risk}
 f_i(\bx)= \frac{1}{m} \sum_{j=1}^m \ell (\bx,\theta_i^j)   ,
\end{align}
where $\cD_{i} = \{\theta_i^1,\cdots,\theta_i^m\}$ is the set of samples available at node $i$. Nodes aim to collaboratively minimize the average of all local objective functions, denoted by $f$, which is given by 
\begin{align}\label{global_cost}
   \min_\bbx f(\bx)= \min_\bbx\frac{1}{n}\sum_{i=1}^n f_i(\bx)
    =\min_\bbx\frac{1}{mn}\sum_{i=1}^n  \sum_{j=1}^m \ell (\bx,\theta_i^j).
\end{align}
Indeed, the objective functions $f$ and $L_N$ are equivalent if $\ccalD \coloneqq \ccalD_1 \cup \cdots \cup \ccalD_n$. Therefore, by minimizing the global objective function $f$ we also obtain the solution of the ERM problem in \eqref{eq:ERM}. 

We can rewrite the optimization problem in \eqref{global_cost} as a classical decentralized optimization problem as follows. Let $\bbx_i$ be the decision variable of node $i$. Then, \eqref{global_cost} is equivalent to 
\begin{equation}\label{cons_prob}
    \min_{\bbx_1,\dots,\bbx_n}  \frac{1}{n}\sum_{i=1}^n f_i(\bx_i), \qquad \text{subject to}\quad  \bbx_1=\dots=\bbx_n,
\end{equation}
as the objective function value of \eqref{global_cost} and \eqref{cons_prob} are the same when the iterates of all nodes are the same and we have \textit{consensus}. The challenge in distributed learning is to solve the global loss only by exchanging information with neighboring nodes and ensuring that nodes' variables stay close to each other.
We consider a network of computing nodes characterized by an undirected connected graph $\mathcal{G}=(\mathcal{V},\mathcal{E})$ with nodes $\mathcal{V}=[n]=\{1,\cdots,n\}$ and edges $\mathcal{E} \subseteq \mathcal{V} \times \mathcal{V}$, and each node $i$ is allowed to exchange information only with its neighboring nodes in the graph $\mathcal{G}$, which we denote by $\ccalN_i$.

In a stochastic optimization setting, where the true objective is defined as an expectation, there is a limit to the accuracy with which we can minimize $L(\bx)$ given only $N=nm$ samples, even if we have access to the optimal solution of the empirical risk $L_N$. In particular, it has been shown that when the loss function $\ell$ is convex,  the difference between the population risk $L$ and the empirical risk $L_N$ corresponding to $N=mn$ samples with high probability is uniformly bounded by $\sup_\bbx |L(\bbx)-L_N(\bbx)| \leq \mathcal{O}(1/\sqrt{N})=\mathcal{O}(1/\sqrt{nm})$; see \citep{bottou2008tradeoffs}.
Thus, without collaboration, each node can minimize its local cost $f_{i}$ to reach an estimate for the optimal solution with an error of $\mathcal{O}(1/\sqrt{m})$. By minimizing the aggregate loss collaboratively, nodes reach an approximate solution of the expected risk problem with a smaller error of  $\mathcal{O}(1/\sqrt{nm})$. 
Based on this formulation, our goal in the convex setting is to find a point $\bbx_i$ for each node $i$ that attains the statistical accuracy, i.e., $\E{L_N(\bbx_i)-L_N(\hbx^*)}\leq \mathcal{O}(1/\sqrt{mn})$, which further  implies $\E{L(\bbx_i)-L(\bbx^*)}\leq \mathcal{O}(1/\sqrt{mn})$. 



For a non-convex loss function $\ell$, however, $L_N$ is also non-convex and solving the problem in \eqref{global_cost} is hard, in general. Therefore, we only focus on finding a point that satisfies the first-order optimality condition for \eqref{global_cost} up to some accuracy  $\rho$, i.e., finding a point $\tbx$ such that $\|\nabla L_N(\tbx)\|=\|\nabla f(\tbx)\|\leq \rho$.  Under the assumption that the gradient of loss is sub-Gaussian,
it has been shown that with high probability the gap between the gradients of expected risk and empirical risk is bounded by $\sup_{\bbx} \|\nabla L(\bbx)-\nabla L_N(\bbx)\|_2 \leq \mathcal{O}( {1}/{\sqrt{nm}})$; see \citep{mei2018landscape}.
As in the convex setting, by solving the aggregate loss instead of local loss, each node finds a better approximate for a first-order stationary point of the expected risk $L$. 
Therefore, our goal in the non-convex setting is to find a point that satisfies $\|\nabla L_N(\bbx)\|\leq \mathcal{O}(1/\sqrt{mn})$ which also implies $\|\nabla L(\bbx)\|\leq \mathcal{O}(1/\sqrt{mn})$.


\section{Proposed \texttt{QuanTimed-DSGD} Method}\label{sec:QT-DSGD}

In this section, we present our proposed \texttt{QuanTimed-DSGD} algorithm that takes into account robustness to stragglers and communication efficiency in decentralized optimization. To ensure robustness to stragglers' delay, we introduce a \textit{deadline-based} protocol for updating the iterates in which nodes compute their local gradients estimation only for a specific amount time and then use their gradient estimates to update their iterates. This is in contrast to the mini-batch setting, in which nodes have to wait for the slowest machine to finish its local gradient computation. 
To reduce the communication load, we assume that nodes only exchange a quantized version of their local iterates. However, using quantized messages induces extra noise in the decision making process which makes the analysis of our algorithm more challenging. A detailed description of the proposed algorithm is as follows.

\noindent\textbf{Deadline-Based Gradient Computation.} Consider the current model $\bx_{i,t}$ available at node $i$ at iteration $t$. Recall the definition of the local objective function $f_i$ at node $i$ defined in \eqref{local_risk}. The cost of computing the local gradient $\nabla f_i$ scales linearly by the number of samples $m$ assigned to the $i$-th node. A common solution to reduce the computation cost at each node for the case that $m$ is large is using a mini-batch approximate of the gradient, i.e., each node $i$ picks a subset of its local samples $\ccalB_{i,t}\subseteq \ccalD_i$ to compute the stochastic gradient $\frac{1}{|\ccalB_{i,t}|}\sum_{\theta\in\ccalB_{i,t}} \! \nabla \ell(\bbx_{i,t},\theta)$. A major challenge for this procedure is the presence of stragglers in the network: 
given mini-batch size $b$, \emph{all} nodes have to compute the average of exactly $b$ stochastic gradients. Thus, all the nodes have to wait for the \emph{slowest} machine to finish its computation and exchange its new model with the neighbors.

To resolve this issue, we propose a deadline-based approach in which we set a fixed deadline $T_d$ for the time that each node can spend computing its local stochastic gradient estimate. Once the deadline is reached, nodes find their gradient estimate using whatever computation (mini-batch size) they could perform. 
Thus, with this deadline-based procedure, nodes do not need to wait for the slowest machine to update their iterates. However, their mini-batch size and consequently the noise of their gradient approximation will be different.  To be more specific, let $\ccalS_{i,t}\subseteq \ccalD_i$ denote the set of random samples chosen at time $t$ by node $i$. Define $\widetilde{\gr} f_i(\bx_{i,t})$ as the stochastic gradient of node $i$ at time $t$ as
\begin{equation}\label{eq:stoch-gr}
    \widetilde{\gr} f_i(\bx_{i,t}) 
    =
    \frac{1}{|\ccalS_{i,t}|} \sum_{\theta\in\ccalS_{i,t}} \gr \ell ( \bx_{i,t}; \theta),
\end{equation}
for $1 \!\leq\! |\ccalS_{i,t}| \!\leq\! m$. If there are not any gradients computed by $T_d$, i.e., $|\ccalS_{i,t}| = 0$, we set $\widetilde{\gr} f_i(\bx_{i,t})\!=\!0$.

\noindent\textbf{Computation Model.} To illustrate the advantage of our deadline-based scheme over the fixed mini-batch scheme, we formally state the model that we use for the processing time of nodes in the network. We remark that our algorithms are oblivious to the choice of the computation model which is merely used for analysis. We define the processing speed of each machine  as the number of stochastic gradients  $\nabla \ell(\bbx,\theta)$ that it computes per second. We assume that the processing speed of each machine $i$ and iteration $t$ is a random variable $V_{i,t}$, and $V_{i,t}$'s are i.i.d. with probability distribution $F_V(v)$. We further assume that the domain of the random variable $V$ is bounded and its realizations are in $[\underbar{$v$},\bar{v}]$. 
If $V_{i,t}$ is the number of stochastic gradient which can be computed per second, the size of mini-batch $\ccalS_{i,t}$ is a random variable given by $|\ccalS_{i,t}|=V_{i,t}T_d$.

In the fixed mini-batch scheme and for any iteration $t$, all the nodes have to wait for the machine with the slowest processing time before updating their iterates, and thus the overall computation time will be $b/V_{\min}$ where $V_{\min}$ is defined as $V_{\min}=\min \{V_{1,t},\dots,V_{n,t}\}$. In our deadline-based scheme there is a fixed deadline $T_d$ which limits the computation time of the nodes, and is chosen 
such that $T_d=\E{b/V}=b\E{1/V}$, while the mini-batch scheme requires an expected time of $\E{b/V_{\min}}=b\E{1/V_{\min}}$. The gap between $\E{1/V}$ and $\E{1/V_{\min}}$ depends on the distribution of $V$, and can be unbounded in general growing with $n$. 




\noindent\textbf{Quantized Message-Passing.} 
To reduce the communication overhead of exchanging variables between nodes, we use quantization schemes that significantly reduces the required number of bits.
More precisely, instead of sending $\bx_{i,t}$, the $i$-th node sends $\bz_{i,t}=Q(\bx_{i,t})$ which is a quantized version of its local variable $\bx_{i,t}$ to its neighbors $j\in \ccalN_i$. As an example, consider the low precision quantizer specified by scale factor $\eta$ and $s$ bits with the representable range~$\{- \eta \cdot 2^{s-1}, \cdots, -\eta,0,\eta, \cdots ,\eta \cdot (2^s-1)\}$. For any $k\eta\le x <(k+1)\eta$ , the quantizer outputs
\begin{equation}\label{eq:quantizerexample}
    Q_{(\eta,b)}(x) 
    =
    \left\{
	\begin{array}{ll}
		k\eta  & \mbox{w.p.  } 1-(x-k\eta)/\eta, \\
	    (k+1)\eta & \mbox{w.p.  }  (x-k\eta)/\eta.
	\end{array}
    \right.
\end{equation}
\begin{algorithm}[t!]
\caption{\texttt{QuanTimed-DSGD} at node $i$}\label{alg:QDC}
\begin{algorithmic}[1]
\Require Weights $\{w_{ij}\}_{j=1}^n$, total iterations $T$, deadline $T_d$
\State Set $\bx_{i,0}=0$ and compute $\bz_{i,0}=Q(\bx_{i,0})$ 
\For{$t=0,\cdots,T-1$}
	\State Send $\bz_{i,t}=Q(\bx_{i,t})$ to $j\in\mathcal{N}_i$ and receive $\bz_{j,t}$
    \parState{Pick and evaluate stochastic gradients $\{ \gr \ell( \bx_{i,t}; \theta) : \theta\in \ccalS_{i,t} \}$ till reaching the deadline $T_d$ and generate $\widetilde{\gr} f_i(\bx_{i,t})$ according to \eqref{eq:stoch-gr}}
    \State Update  $\bx_{i,t+1}$ as follows:
       $ \bx_{i,t+1} =  (1 - \eps + \eps w_{ii}) \bx_{i,t} + \eps \sum_{ j \in  \mathcal{N}_i } w_{ij} \bz_{j,t} - \al \eps \widetilde{\gr} f_i(\bx_{i,t})$
\EndFor
\end{algorithmic}
\end{algorithm}
\textbf{Algorithm Update.} Once the local variables are exchanged between neighboring nodes, each node~$i$ uses its local stochastic gradient $\widetilde{\gr} f_i(\bx_{i,t})$, its local decision variable $\bx_{i,t}$, and the information received from its neighbors  $\{ \bz_{j,t}=Q(\bx_{j,t}); j \in \mathcal{N}_i \}$ to update its local decision variable. Before formally stating the update of \texttt{QuanTimed-DSGD}, let us define $w_{ij}$ as the weight that node $i$ assigns to the information that it receive from node $j$. If $i$ and $j$ are not neighbors $w_{ij}=0$. These weights are considered for averaging over the local decision variable $x_{i,t}$ and the quantized variables $ \bz_{j,t}$ received from neighbors to enforce consensus among neighboring nodes. Specifically, at time $t$, node~$i$ updates its decision variable according to the update
\begin{align}\label{update_of_the_alg}
    \bx_{i,t+1} =  (1 - \eps + \eps w_{ii}) \bx_{i,t} + \eps  \sum_{ j \in  \mathcal{N}_i } w_{ij} \bz_{j,t} - \al \eps \widetilde{\gr} f_i(\bx_{i,t}),
\end{align}
where $\al$ and $\eps$ are positive scalars that behave as stepsize. Note that the update in \eqref{update_of_the_alg} shows that the updated iterate is a linear combination of the weighted average of node $i$'s  neighbors' decision variable, i.e.,  $\eps\sum_{ j \in  \mathcal{N}_i } w_{ij} \bz_{j,t}$, and its local variable $\bbx_{i,t}$ and stochastic gradient $\widetilde{\gr} f_i(\bx_{i,t})$. The parameter $\alpha$ behaves as the stepsize of the gradient descent step with respect to local objective function and the parameter $\eps$ behaves as an averaging parameter between performing the distributed gradient update  $\eps( w_{ii}\bx_{i,t} +  \sum_{ j \in  \mathcal{N}_i } w_{ij} \bz_{j,t} - \al  \widetilde{\gr} f_i(\bx_{i,t}))$ and using the previous decision variable $(1-\eps)\bx_{i,t}$. By choosing a diminishing stepsize $\alpha$ we control the noise of stochastic gradient evaluation, and by averaging using the parameter $\eps$ we control randomness induced by exchanging quantized variables.
The description of \texttt{QuanTimed-DSGD} is summarized in Algorithm~\ref{alg:QDC}.


\section{Convergence Analysis}\label{sec:convergence}

In this section, we provide the main theoretical results for the proposed  \texttt{QuanTimed-DSGD} algorithm. We first consider strongly convex loss functions and characterize the convergence rate of \texttt{QuanTimed-DSGD} for achieving the global optimal solution to the problem \eqref{global_cost}. Then, we focus on the non-convex setting and show that the iterates generated by \texttt{QuanTimed-DSGD} find a stationary point of the cost in \eqref{global_cost} while the local models are close to each other and the consensus constraint is asymptotically satisfied. All the proofs are provided in the supplementary material (Section 6).  
We make the following assumptions on the weight matrix, the quantizer, and local objective functions. 

\begin{assumption}\label{assump-W}
The weight matrix $W \in \reals^{n\times n}$ with entries $w_{ij} \geq 0$ satisfies the following conditions: $W = W^\top$, $W \mathbf{1}=\mathbf{1}$ and $\text{null}(I-W)= \text{span}(\mathbf{1})$.
\end{assumption}

\begin{assumption}\label{assump-Q}
The random quantizer $Q(\cdot)$ is unbiased and variance-bounded, i.e.,
$ \bE [Q(\bx)| \bx ]=\bx$ and $ \bE [ \|Q(\bx)-\bx\|^2 | \bx ]\leq \sigma^2,$
for any $\bx \in \reals^{p}$; and quantizations are carried out independently.
\end{assumption}

Assumption \ref{assump-W} implies that $W$ is symmetric and doubly stochastic. Moreover, all the eigenvalues of $W$ are in $(-1,1]$, i.e., $1 = \lambda_1(W) \geq \lambda_2(W) \geq \cdots \geq \lambda_n(W) > -1$ (e.g. \citep{yuan2016convergence}). We also denote by $1-\beta$ the spectral gap associated with the stochastic matrix $W$, where $\beta=\max \left\{|\lambda_2(W)|,|\lambda_n(W)| \right\}$. 

\begin{assumption}\label{assump-smooth}
The function $\ell$ is $K$-smooth with respect to $\bx$, i.e., for any $\bx, \hbx \in \reals^{p}$ and any $\theta\in \ccalD$,
$ \norm{\nabla  \ell(\bx,\theta) - \gr \ell(\hbx,\theta)} \leq K \norm{\bx - \hbx}$.
\end{assumption}

\begin{assumption}\label{assump-gr}
Stochastic gradients $\gr \ell(\bx, \theta)$ are unbiased and variance bounded, i.e., 
$\mathbb{E}_{\theta} \left[ \gr \ell(\bx, \theta) \right] =  \gr L(\bx)$ and $ \mathbb{E}_{\theta} \left[\ \norm{ \gr \ell(\bx, \theta) - \gr L(\bx)}^2 \right] \leq \gamma^2.$
\end{assumption}
Note the condition in Assumption \ref{assump-gr} implies that the local gradients of each node $\nabla f_i(\bbx)$ are also unbiased estimators of the expected risk gradient $\nabla L (\bbx)$ and their variance is bounded above by $\gamma^2/m$ as it is defined as an average over $m$ realizations.

\subsection{Strongly Convex Setting}

This section presents the convergence guarantees of the proposed \texttt{QuanTimed-DSGD} method for smooth and strongly convex functions. The following assumption formally defines strong convexity.

\begin{assumption}\label{assump-convex}
The function $\ell$ is $\mu$-strongly convex, i.e., for any $\bx, \hbx \in \reals^{p}$ and $\theta\in \ccalD$ we have that  
$   \langle \gr \ell (\bx, \theta) - \gr \ell (\hbx, \theta), \bx - \hbx  \rangle \geq \mu \norm{\bx - \hbx}^2.
$
\end{assumption}

Next, we characterize the convergence rate of \texttt{QuanTimed-DSGD}  for strongly convex objectives.

\vspace{2mm}
\begin{theorem}[Strongly Convex Losses]\label{thm1}
If the conditions in Assumptions \ref{assump-W}--\ref{assump-convex} are satisfied and step-sizes are picked as $\al = {T^{-\delta/2}}$ and $\eps = {T^{-3\delta/2}}$ for arbitrary $\delta \in (0,1/2)$, then for large enough number of iterations $T \geq T^{\mathsf{c}}_{\mathsf{min}}$ the iterates generated by the \texttt{QuanTimed-DSGD} algorithm satisfy
\begin{equation}\label{cvx_main_result}
    {{\frac{1}{n}}} \sum_{i=1}^{n} \mbE \left[ \norm{\bx_{i,T} \!-\! {\bx}^*}^2 \right] 
     \!\leq\!
    \ccalO\! \left( \frac{D^2 (K/\mu)^2 }{(1-\beta)^2} + \frac{\sigma^2}{\mu} \right) \!\frac{1}{T^{\delta}}  +
    \ccalO \!\left( \frac{\gamma^2}{\mu}  \max{\left\{\frac{\mbE[1/V]}{T_d},  \frac{1}{m}\right\}} \right)\! \frac{1}{T^{2\delta}},
\end{equation}
where $D^2 =2 K  \sum_{i=1}^{n}( f_i(0) - f^*_i )$, and $ f^*_i =  \min_{\bx \in \mathbb{R}^p} f_i(\bx)$. 
\end{theorem}

Theorem \ref{thm1} guarantees the \emph{exact} convergence of \emph{each} local model to the global optimal even though the noises induced by random quantizations and stochastic gradients are non-vanishing with iterations. Moreover, such convergence rate is as close as desired to $\mathcal{O}( {1}/{\sqrt{T}})$ by picking the tuning parameter $\delta$ arbitrarily close to $1/2$. We would like to highlight that by choosing a parameter $\delta$ closer to $1/2$, the lower bound on the number of required iterations $T^{\mathsf{c}}_{\mathsf{min}}$ becomes larger. More details are available in the proof of Theorem~\ref{thm1} provided in the supplementary material.

Note that the coefficient of $1/T^{\delta}$ in \eqref{cvx_main_result}  characterizes the dependency of our upper bound on the objective function condition number $K/\mu$, graph connectivity parameter $1/(1-\beta)$, and variance $\sigma^2$ of error induced by quantizing our signals. Moreover, the coefficient of $1/T^{2\delta}$ shows the effect of stochastic gradients variance $\gamma^2$ as well as our deadline-based scheme parameters $T_d/(\mbE[1/V])$. 

\vspace{1mm}
\begin{remark}
The expression $1/b_{\text{eff}}=\max\{\mbE[1/V]/T_d,1/m\}$ represents the inverse of the effective batch size $b_{\text{eff}}$ used in our \texttt{QuanTimed-DSGD} method. To be more specific, If the deadline $T_d$ is large enough that in expectation all local gradients are computed before the deadline, i.e., $T_d/\mbE[1/V] >m$, then our effective batch size is $b_{\text{eff}}=m$ and the term $1/m$ is the dominant term in the maximization. Conversely, if $T_d$ is small and the number of computed gradients $T_d/\mbE[1/V]$ is smaller than the total number of local samples $m$, the effective batch size is $b_{\text{eff}}=T_d/\mbE[1/V]$. In this case, $\mbE[1/V]/T_d$ is dominant term in the maximization. This observation shows that  $\gamma^2\max\{\mbE[1/V]/T_d,1/m\}=\gamma^2/b_{\text{eff}}$ in \eqref{cvx_main_result} is the variance of mini-batch gradient in \texttt{QuanTimed-DSGD}.
\end{remark}

\vspace{1mm}
\begin{remark}
Using strong convexity of the objective function, one can easily verify that the last iterates $\bbx_{i,T}$ of  \texttt{QuanTimed-DSGD} satisfy the sub-optimality 
$f(\bbx_{i,T})-f(\hbx^*)= L_N(\bbx_{i,T})-L_N(\hbx^*) \leq \mathcal{O}(1/\sqrt{T})$ with respect to the empirical risk, where $\hbx^*$ is the minimizer of the empirical risk $L_N$. As the gap between the expected risk $L$ and the empirical risk $L_N$ is of $\mathcal{O}(1/\sqrt{mn})$, the overall error of  \texttt{QuanTimed-DSGD} with respect to the expected risk $L$ is $\mathcal{O}(1/\sqrt{T}+{1}/{\sqrt{mn}})$. 
\end{remark}

\subsection{Non-convex Setting}

In this section, we characterize the convergence rate of \texttt{QuanTimed-DSGD} for non-convex and smooth objectives. As discussed in Section \ref{sec:setup}, we are interested in finding a set of local models which satisfy first-order optimality condition approximately, while the models are close to each other and satisfy the consensus condition up to a small error. To be more precise, we are interested in finding a set of local models $\{\bx_1^*,\dots,\bx_n^*\}$ where their average  $\overline{\bbx}^* \coloneqq \frac{1}{n}\sum_{i=1}^n \bbx_{i}^*$ (approximately) satisfy first-order optimality condition, i.e., $\bE \norm{ \gr f \left( \overline{\bbx}^*\right)}^2\leq \nu$, while the iterates are close to their average, i.e., $\bE \|\overline{\bbx}^*-\bx_{i}^*\|^2 \leq \rho$.
If a set of local iterates satisfies these conditions we call them $(\nu,\rho)$-approximate solutions. 
Next theorem characterizes both first-order optimality and consensus convergence rates and the overall complexity for achieving an $(\nu,\rho)$-approximate solutions.

\vspace{2mm}
\begin{theorem}[Non-convex Losses]\label{thm2}
Under Assumptions \ref{assump-W}--\ref{assump-gr}, and for step-sizes $\al = T^{-1/6}$ and $\eps = T^{-1/2}$, \texttt{QuanTimed-DSGD} guarantees the following convergence and consensus rates:
\begin{equation}\label{eq:thm2-convergence}
    \frac{1}{T} \sum_{t=0}^{T-1} \bE \norm{ \gr f (\overline{\bbx}_t )}^2 
    \leq
    \ccalO \left( \frac{ K^2 }{(1 - \beta)^2} \frac{\gamma^2}{m} + \frac{K \sigma^2}{n} \right) \frac{1}{T^{1/3}}
   +
    \ccalO \left(  \frac{K\gamma^2}{n} \max{\left\{\frac{\mbE[1/V]}{T_d},  \frac{1}{m}\right\}} \right) \frac{1}{T^{2/3}},
\end{equation}
and
\begin{equation}\label{eq:thm2-consensus}
    \frac{1}{T} \sum_{t=0}^{T-1} \frac{1}{n} \sum_{i=1}^{n} \bE \norm{ \overline{\bbx}_t  - \bx_{i,t} }^2
    \leq
    \ccalO \left( \frac{ \gamma^2}{m (1 - \beta)^2} \right) \frac{1}{T^{1/3}},
\end{equation}
for large enough number of iterations $T \geq T^{\mathsf{nc}}_{\mathsf{min}}$. Here $\overline{\bbx}_t = \frac{1}{n} \sum_{i=1}^{n} \bbx_{i,t}$ denotes the average models at iteration $t$.
\end{theorem}

The convergence rate in \eqref{eq:thm2-convergence} indicates the proposed \texttt{QuanTimed-DSGD} method finds first-order stationary points with vanishing approximation error, even though the quantization and stochastic gradient noises are non-vanishing. Also, the approximation error decays as fast as $\ccalO ( T^{-1/3} )$ with iterations. Theorem \ref{thm2} also implies from \eqref{eq:thm2-consensus} that the local models reach consensus with a rate of  $\ccalO ( T^{-1/3} )$. Moreover, it shows that to find an  $(\nu,\rho)$-approximate solution \texttt{QuanTimed-DSGD} requires at most  $\mathcal{O}(\max\{\nu^{-3},\rho^{-3}\})$ iterations.

\section{Experimental Results}\label{sec:numerical}
\vspace{0mm}
In this section, we numerically evaluate the performance of the proposed \texttt{QuanTimed-DSGD} method described in Algorithm \ref{alg:QDC} for solving a class of non-convex decentralized optimization problems. In particular, we compare the total run-time of \texttt{QuanTimed-DSGD} scheme with the ones for three benchmarks which are briefly described below. 
\begin{itemize}[leftmargin=*]
\item \textbf {Decentralized SGD (DSGD)} \citep{yuan2016convergence}:  Each worker updates its decision variable as
$\bx_{i,t+1} = \sum_{ j \in  \mathcal{N}_i } w_{ij} \bx_{j,t} - \alpha \widetilde{\gr} f_i(\bx_{i,t}) $. 
We note that the exchanged messages are not quantized and the local gradients are computed for a fixed batch size. 
\item \textbf{Quantized Decentralized SGD (Q-DSGD)} \citep{reisizadeh2019exact}: Iterates are updated according to \eqref{update_of_the_alg}. Similar to \texttt{QuanTimed-DSGD} scheme, Q-DSGD employs quantized message-passing, however the gradients are computed for a fixed batch size in each iteration.
\item \textbf{Asynchronous DSGD}: Each worker updates its model without waiting to receive the updates of 
its neighbors, i.e. $\bx_{i,t+1} = \sum_{ j \in  \mathcal{N}_i } w_{ij} \bx_{j,\tau_j} - \alpha \widetilde{\gr} f_i(\bx_{i,t}) $ where $\bx_{j,\tau_j}$ denotes the most recent model for node $j$. In our implementation of this scheme, models are exchanged without quantization.
\end{itemize}

Note that the first two methods mentioned above, i.e., DSGD and Q-DSGD, operate synchronously across the workers, as is our proposed \texttt{QuanTimed-DSGD} method. To be more specific, worker nodes wait to receive the decision variables from all of the neighbor nodes and then synchronously update according to an update rule. In \texttt{QuanTimed-DSGD} (Figure \ref{fig:asynch}, right), this waiting time consists of a fixed gradient computation time denoted by the deadline $T_d$ and  communication time of the message exchanges. Due to the random computation times, different workers end up computing gradients of different and random batch-sizes $B_{i,t}$ across workers $i$ and iterations $t$. In DSGD (and Q-DSGD) however (Figure \ref{fig:asynch}, Left), the gradient computation time varies across the workers since computing a fixed-batch gradient of size $B$ takes a random time whose expected value is proportional to the batch-size $B$ and hence the slowest nodes (stragglers) determine the overall synchronization time $T_{\text{max}}$. Asynchronous-DSGD  mitigates stragglers since each worker iteratively computes a gradient of batch-size $B$ and updates the local model using the most recent models of its neighboring nodes available in its memory (Figure \ref{fig:asynch}, middle).

\begin{figure}[h!]
    \centering
    \includegraphics[width=0.7\linewidth]{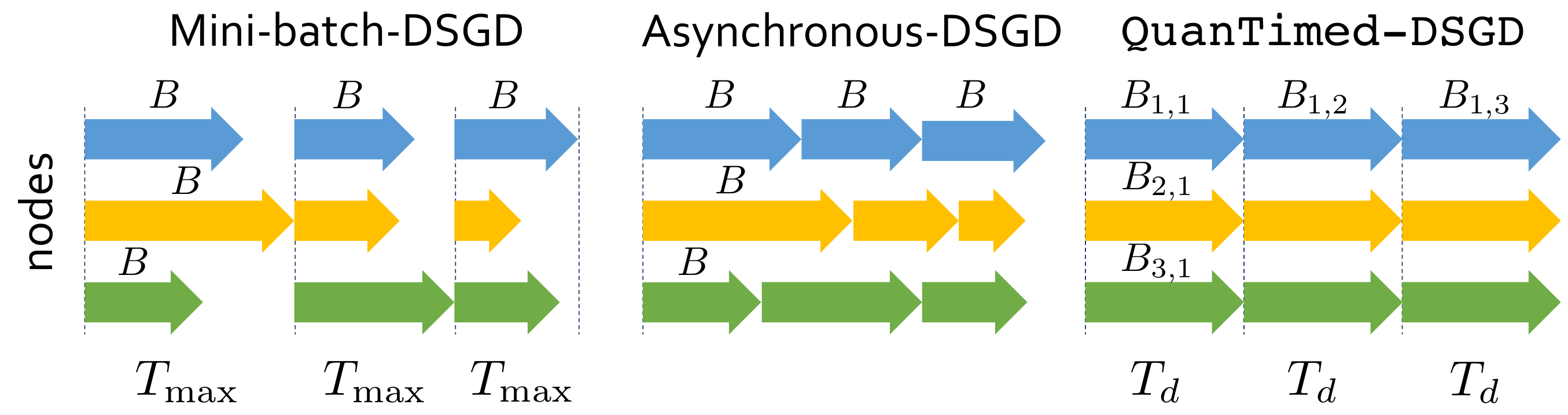}
    \vspace{0mm}
    \caption{Gradient computation timeline for three methods: DSGD, Asynchronous-DSGD, \texttt{QuanTimed-DSGD}.}
    \label{fig:asynch}
    \vspace{-1mm}
\end{figure}
\textbf{Data and Experimental Setup.} We carry out two sets of experiments over CIFAR-10 and MNIST datasets, where each worker is assigned with a sample set of size $m = 200$ for both datasets.  For CIFAR-10, we implement a binary classification using a fully connected neural network with one hidden layer with $30$ neurons.  Each image is converted to a vector of length $1024$.
For MNIST, we use a fully connected neural network with one hidden layer of size $50$ to classify the input image into $10$ classes. 
In experiments over CIFAR-10, step-sizes are fine-tuned as follows: $(\alpha, \eps) = ({0.08}/{T^{1/6}},{14}/{T^{1/2}}) $ for \texttt{QuanTimed-DSGD} and Q-DSGD, and $\alpha = 0.015$ for DSGD and Asynchronous DSGD. In MNIST experiments, step-sizes are fine-tuned to $(\alpha, \eps) = ({0.3}/{T^{1/6}},{15}/{T^{1/2}})$ for \texttt{QuanTimed-DSGD} and Q-DSGD, and $\alpha = 0.2$ for DSGD.

We implement the unbiased low precision quantizer in \eqref{eq:quantizerexample} with various quantization levels $s$, and we let $T_c$ denote the communication time of a $p$-vector without quantization (16-bit precision). The communication time for a quantized vector is then proportioned according the quantization level. 
In order to ensure that the expected batch size used in each node is a target positive number $b$, we choose the deadline $T_d = b /\mathbb{E}[V]$, where $V \sim \text{Uniform}(10,90)$ is the random computation speed. 
The communication graph is a random Erdös-Rènyi graph with edge connectivity $p_c = 0.4$ and $n=50$ nodes. The weight matrix is designed as $\bbW = \bbI -  {\bbL}/\kappa$ where $\bbL$ is the Laplacian matrix of the graph and $ \kappa > \lambda_{\text{max}}(\bbL)/2$.


\begin{figure}[t!]
    \centering
    \hspace{0mm}\includegraphics[width=0.48\linewidth]{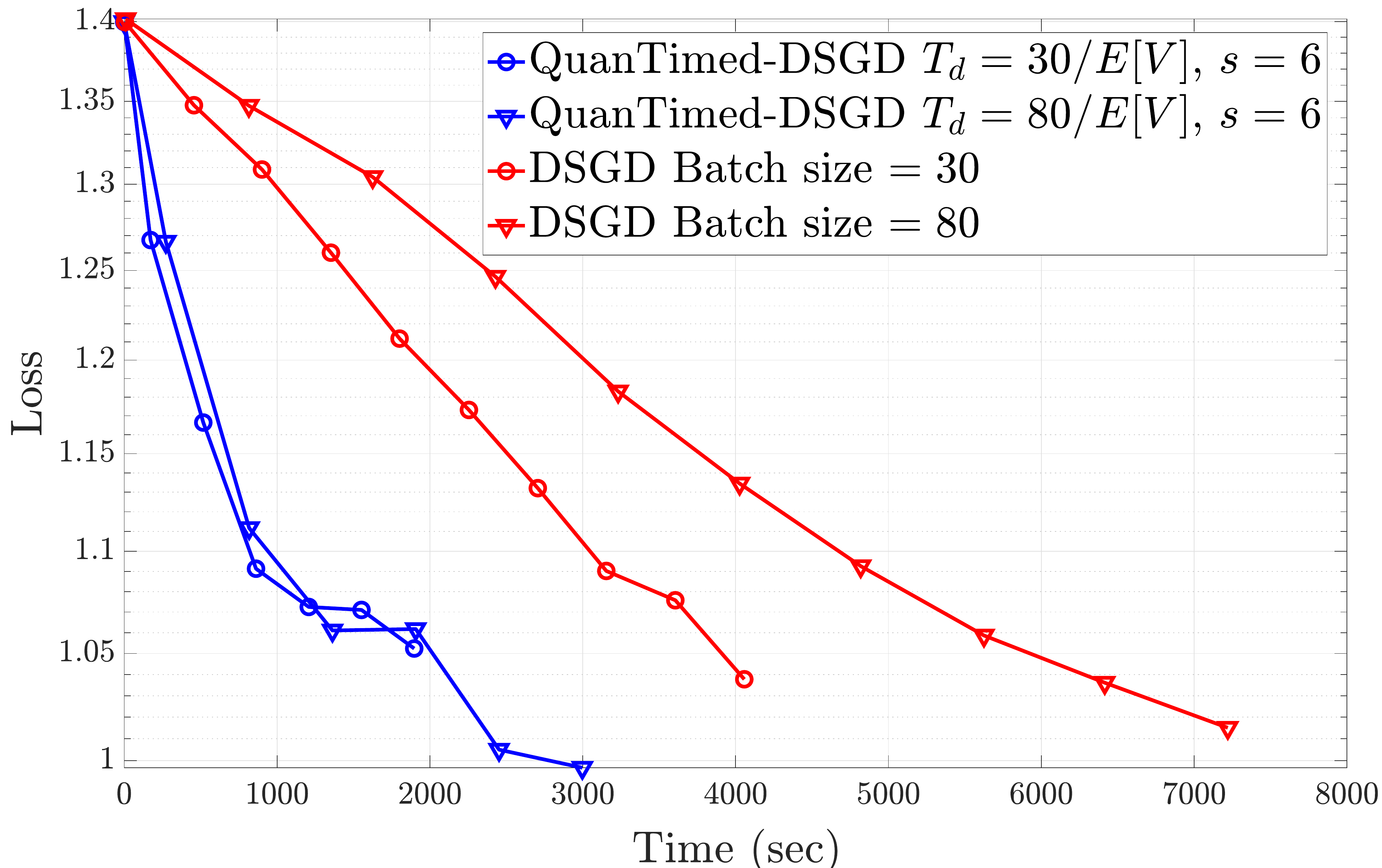}
    \includegraphics[width=0.48\linewidth]{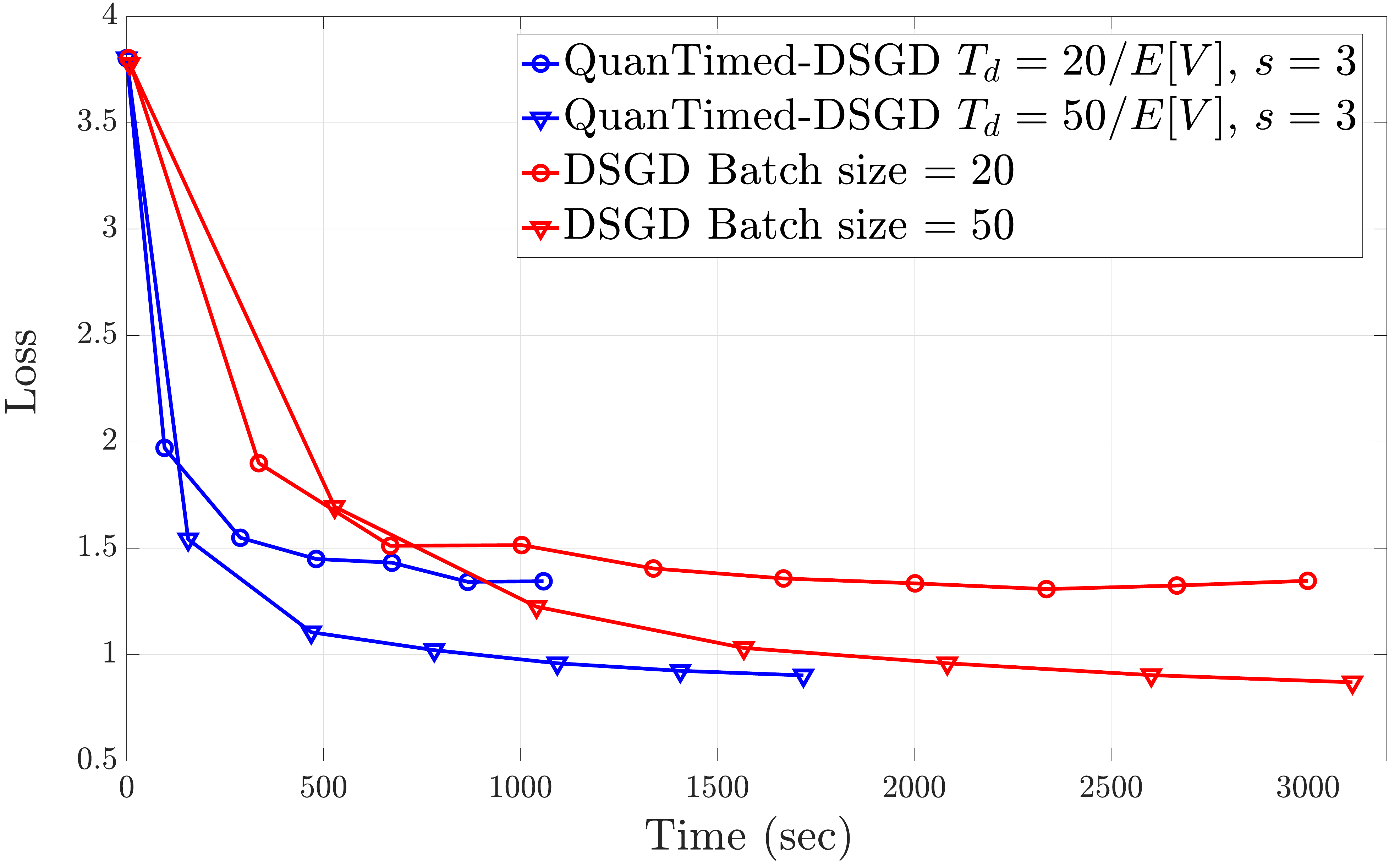}
    \vspace{0mm}
    \caption{Comparison of \texttt{QuanTimed-DSGD} and vanilla DSGD methods for training a neural network on CIFAR-10 (left) and MNIST (right) datasets ($T_c=3$).}
    \label{fig1}
    \vspace{-3mm}
\end{figure}

\begin{figure}[t!]
    \centering
    \hspace{0mm}\includegraphics[width=0.48\linewidth]{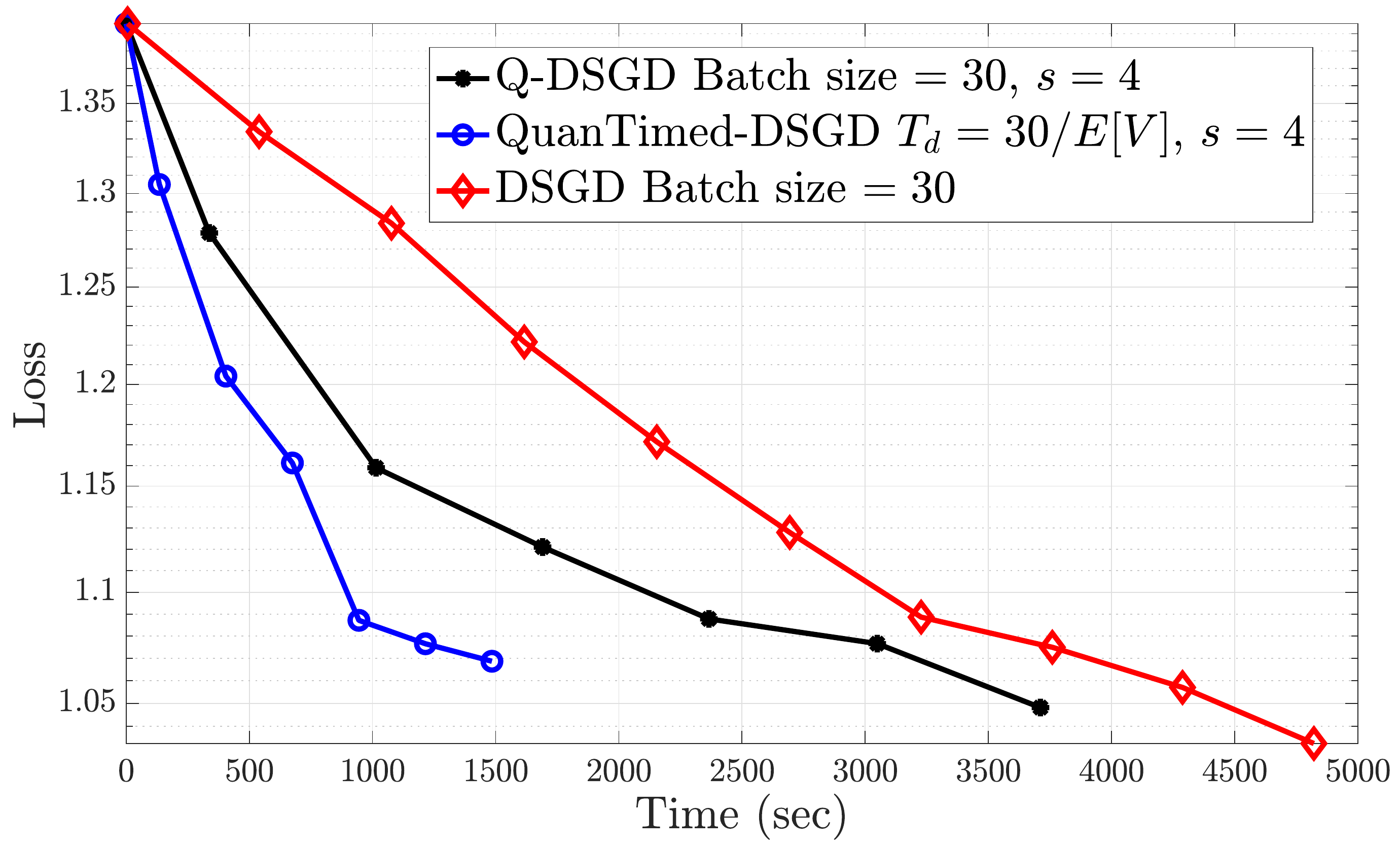}
    \includegraphics[width=0.485\linewidth]{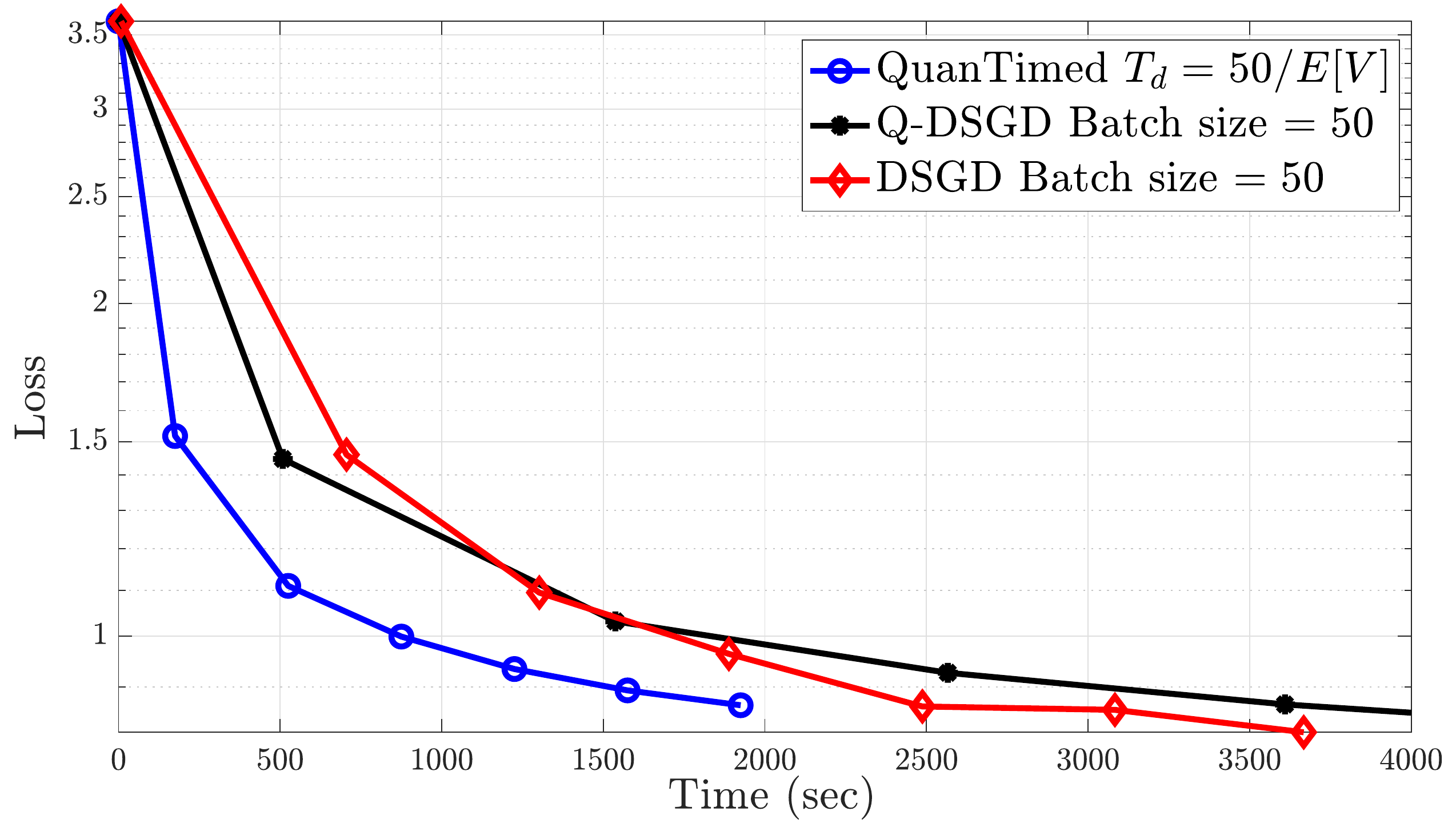}\vspace{0mm}
    \caption{Comparison of \texttt{QuanTimed-DSGD}, QDSGD, and vanilla DSGD methods for training a neural network on CIFAR-10 (left) and MNIST (right) datasets ($T_c=3$).}
    \label{fig2}
    \vspace{-4mm}
\end{figure}

\begin{figure}[t!]
    \centering
    \hspace{0mm}\includegraphics[width=0.47\linewidth]{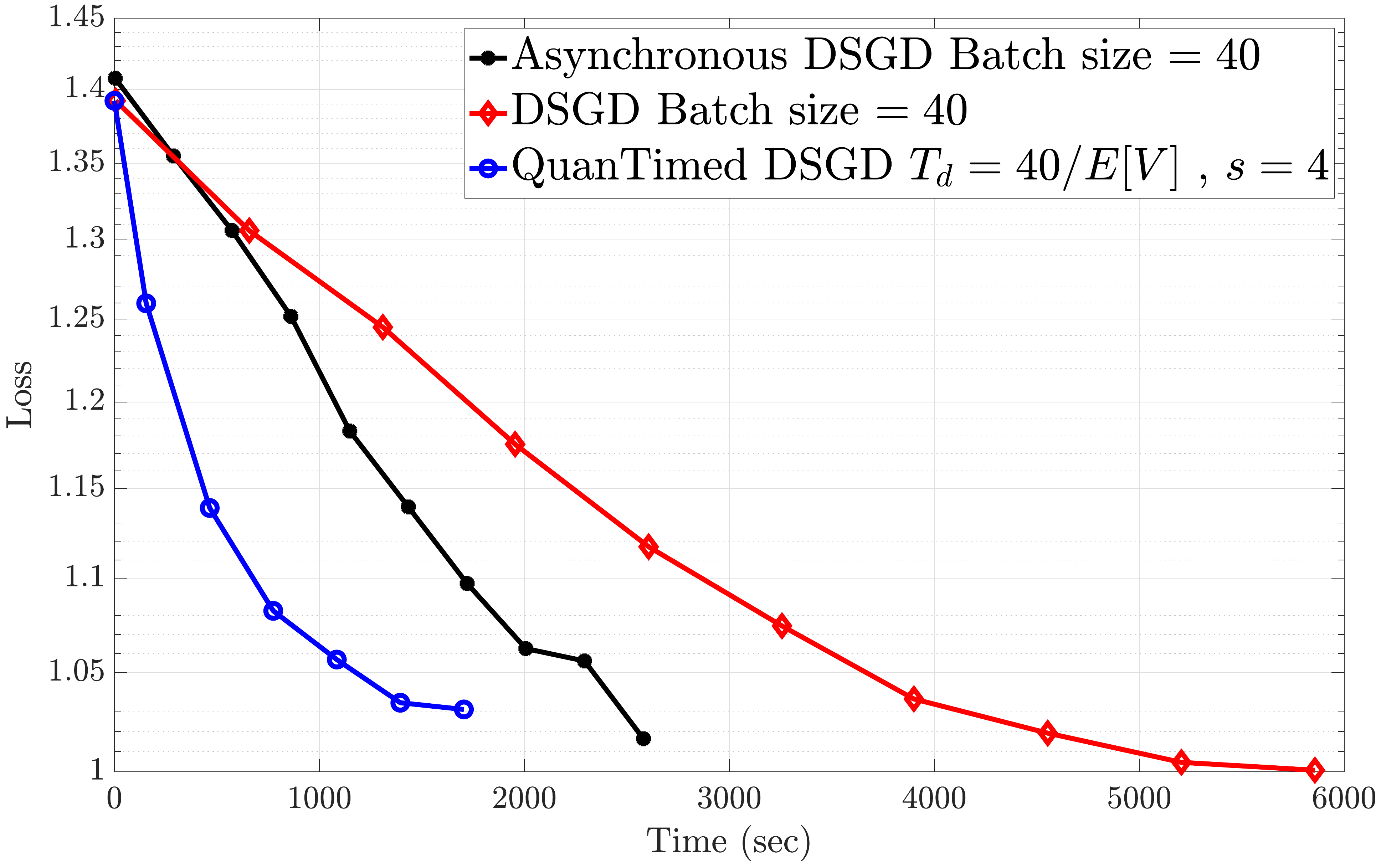}
    \hspace{0mm}\includegraphics[width=0.475\linewidth]{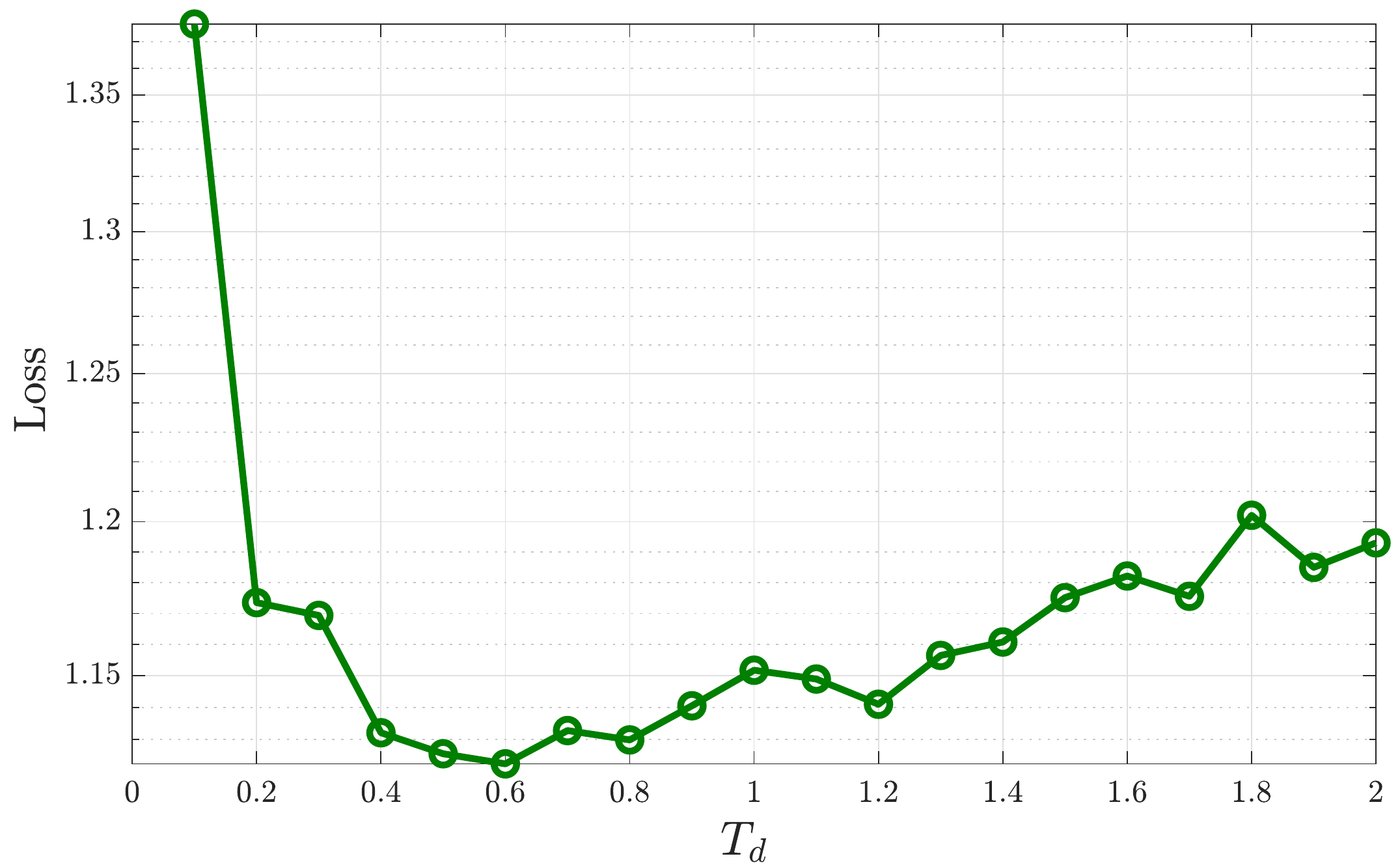}
    \caption{Left: Comparison of \texttt{QuanTimed-DSGD} with Asynchronous DSGD and DSGD for training a neural network on CIFAR-10 ($T_c=3$).  Right: Effect of $T_d$ on the loss for CIFAR-10 ($T_c=1$). }
    \label{fig3}
    \vspace{-6mm}
\end{figure}


\noindent \textbf{Results.} 
Figure \ref{fig1} compares the total training run-time for the \texttt{QuanTimed-DSGD} and DSGD schemes. On CIFAR-10 for instance (left), the same (effective) batch-sizes, the proposed \texttt{QuanTimed-DSGD} achieves speedups of up to $3 \times$ compared to DSGD. 

In Figure \ref{fig2}, we further compare these two schemes to Q-DSGD benchmark. Although Q-SGD improves upon the vanilla DSGD by employing quantization, however, the proposed \texttt{QuanTimed-DSGD} illustrates $ 2 \times$ speedup in training time over Q-DSGD (left). 

To evaluate the straggler mitigation in the \texttt{QuanTimed-DSGD}, we compare its run-time with Asynchronous DSGD benchmark in Figure \ref{fig3} (left). While Asynchronous DSGD outperforms DSGD in training run-time by avoiding slow nodes, the proposed \texttt{QuanTimed-DSGD} scheme improves upon Asynchronous DSGD by up to $30 \%$. These plots further illustrate that \texttt{QuanTimed-DSGD} significantly reduces the training time by simultaneously handling the communication load by quantization and mitigating stragglers through a deadline-based computation. The deadline time $T_d$ indeed can be optimized for the minimum training run-time, as illustrated in Figure \ref{fig3} (right). Additional numerical results on neural networks with four hidden layers and ImageNet dataset are provided in the supplementary materials.

\section{Acknowledgments}\label{sec:acl}
The authors acknowledge supports from National Science Foundation (NSF) under grant CCF-1909320 and UC Office of President under Grant LFR-18-548175. The research of H. Hassani is supported by NSF grants 1755707 and 1837253.

{\small{
\newpage
\bibliography{biblio,bmc_article,bmc_article2}
\bibliographystyle{neurips_2019}
}}

\newpage
\section{Supplementary Material}\label{sec:supp}

\subsection{Additional Numerical Results}\label{sec:extra-numerical}
In this section, we provide additional numerical results comparing the performance of the proposed \texttt{QuanTimed-DSGD} method with other benchmarks, DSGD and Q-DSGD. In particular, we train a fully connected neural network with four hidden layers consisting of $(30, 20, 20, 25)$ neurons on two classes of CIFAR-10 dataset. For \texttt{QuanTimed-DSGD} and Q-DSGD, step sizes are fine-tuned to $(\alpha,\eps) = (0.8/T^{1/6},9/T^{1/2})$ and $(\alpha,\eps) = (0.8/T^{1/6},10/T^{1/2})$, respectively and for DSGD algorithm $\alpha = 0.08$. Figure \ref{fig:extra-sim} (left) demonstrates the training time for the three methods where the proposed \texttt{QuanTimed-DSGD} method enjoys a $2.31 \times$ speedup over the best of the two benchmarks DSGD and Q-DSGD.

Moreover, we use a one hidden layer neural network with $90$ neurons for binary classification on ImageNet dataset and demonstrate the speedup of our proposed method over other benchmarks. For \texttt{QuanTimed-DSGD} and Q-DSGD, step sizes are fine-tuned to $(\alpha,\eps) = (0.1/T^{1/6},14/T^{1/2})$ and $(\alpha,\eps) = (0.1/T^{1/6},15/T^{1/2})$, respectively and for DSGD algorithm $\alpha = 0.02$. Figure \ref{fig:extra-sim} (right) shows the training time of the three methods where \texttt{QuanTimed-DSGD} demonstrates $1.7 \times$ speedup compared to the best of DSGD and Q-DSGD.




\begin{figure}[h]
    \centering
    \includegraphics[width=0.495\textwidth]{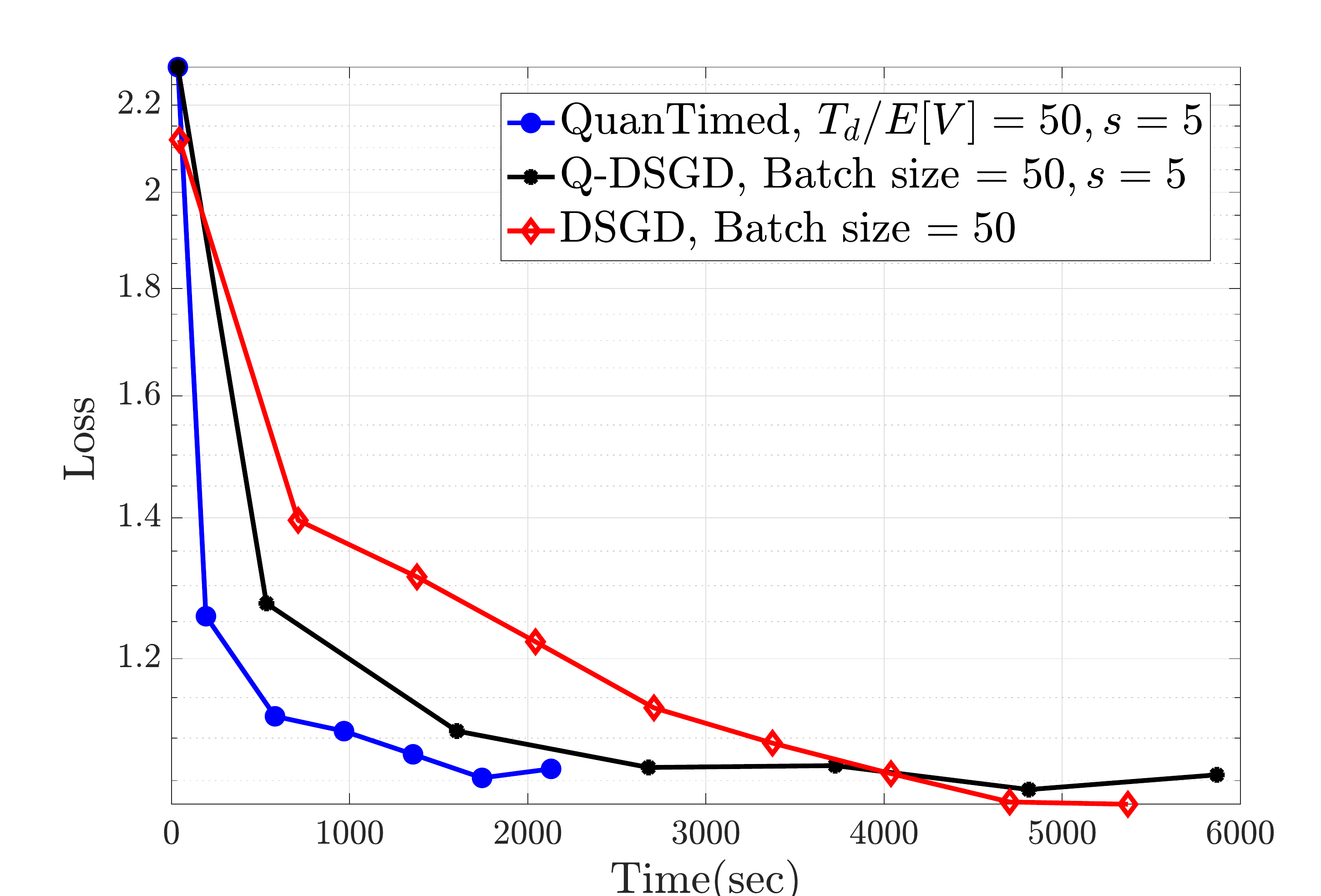}
    \includegraphics[width=0.485\textwidth]{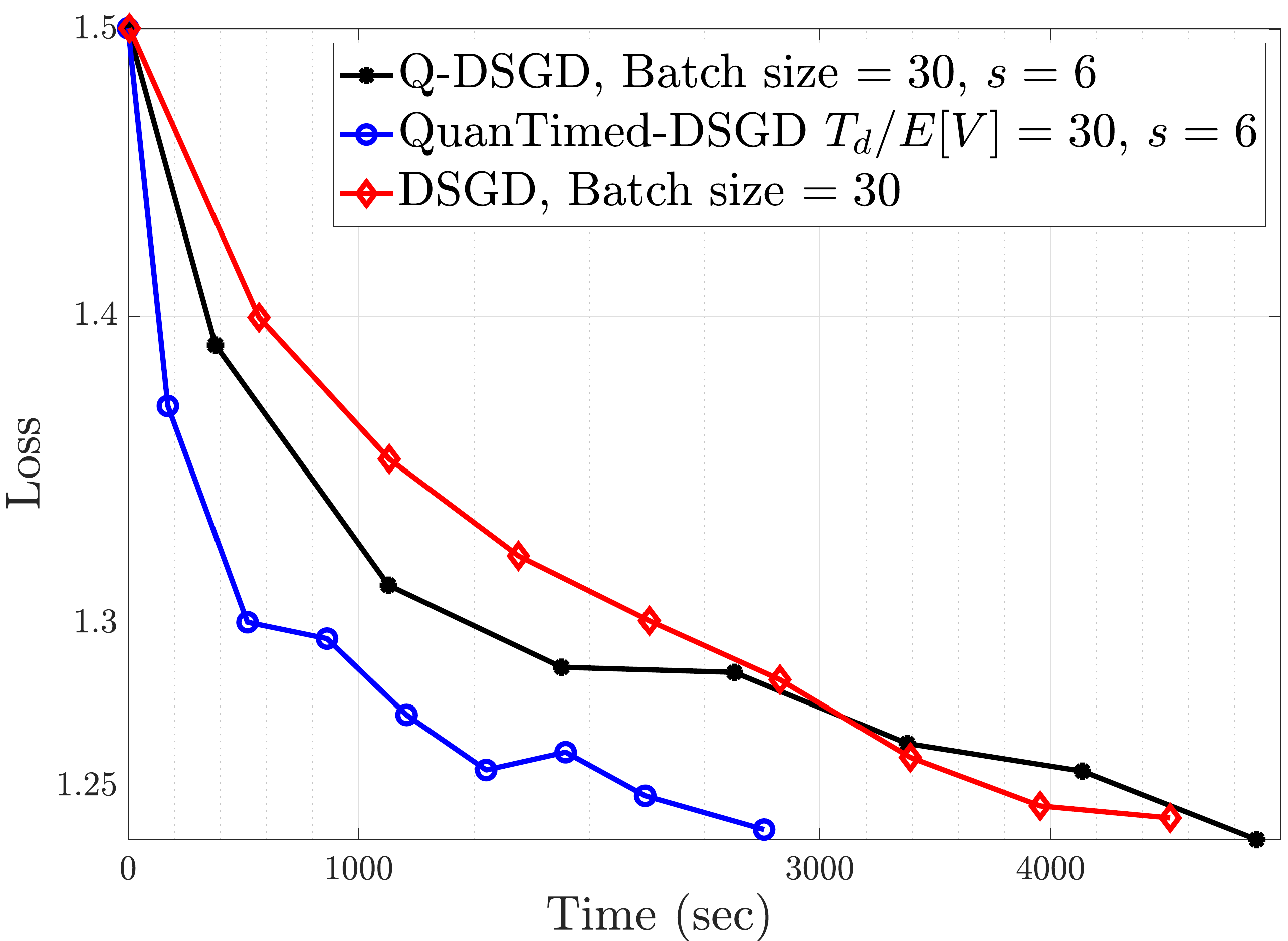}
    \caption{Comparison of \texttt{QuanTimed-DSGD} with Q-DSGD and DSGD for training a four layer neural network on CIFAR-10 dataset (left) and for training a one layer neural network on ImageNet dataset (right). }
    \label{fig:extra-sim}
\end{figure}

\subsection{Bounding the Stochastic Gradient Noises}\label{sec:stoch-g-noise}

In our analysis for both convex and non-convex scenarios, we need to have the noise of various stochastic gradient functions evaluated. Hence, let us start this section by the following lemma which bounds the variance of stochastic gradient functions under our customary Assumption \ref{assump-gr}.

\begin{lemma} \label{lemma:stoch-bound}
Assumption \ref{assump-gr} results in the followings for any $\bx \in \reals^p$ and $i \in [n]$:
\begin{enumerate}[label=(\roman*)]
    \item $\mbE [\gr f_i(\bx)] = \mbE [\gr f(\bx)] = \gr L(\bx)$
    \item $\mbE \left[ \norm{ \gr f_i(\bx) - \gr L(\bx)}^2 \right] \leq \frac{\gamma^2}{m}$
    \item $\mbE \left[ \norm{ \gr f(\bx) - \gr L(\bx)}^2 \right] \leq \frac{\gamma^2}{nm}$
    \item $\mbE \left[ \norm{ \gr f_i(\bx) - \gr f(\bx)}^2 \right] \leq \gamma_1^2 \coloneqq \gamma^2 \left( \frac{1}{m} + \frac{1}{nm} \right)$
    \item $\mbE \left[ \tNab f_i(\bx) \right] =  \gr L(\bx)$
    \item $\mbE \left[ \norm{ \tNab f_i(\bx) - \gr f_i(\bx)}^2 \right] \leq \gamma_2^2 \coloneqq 2 \gamma^2 \max{\left\{\frac{\mbE[1/V]}{T_d},  \frac{1}{m}\right\}}$
\end{enumerate}
\end{lemma}
\begin{proof}
The first five expressions (i)-(v) in the lemma are immediate results of Assumption \ref{assump-gr} together with the fact that the noise of the stochastic gradient scales down with the sample size. To prove (vi), let $\ccalS_{i}$ denote the sample set for which node $i$ has computed the gradients. We have
\begin{align}
    \mbE \left[ \norm{ \tNab f_i(\bx) - \gr L(\bx)}^2 \right]
    & = 
    \sum_{b} \Pr{|\ccalS_{i}| = b} \mbE \norm{ \frac{1}{b} \sum_{\theta\in\ccalS_{i}} \gr \ell ( \bx; \theta) - \gr L(\bx)}^2 \nonumber \\
    & \leq
    \gamma^2 \sum_{b} \Pr{|\ccalS_{i}|=b} \frac{1}{b} \nonumber \\
    & =
    \gamma^2 \mbE[1/|\ccalS_{i}|]  \nonumber \\
    & =
    \gamma^2 \frac{\mbE[1/V]}{T_d},  \nonumber
\end{align}
and therefore
\begin{align}
    \mbE \left[ \norm{ \tNab f_i(\bx) - \gr f_i(\bx)}^2 \right]
    & = 
    \mbE \left[ \norm{ \tNab f_i(\bx) - \gr L(\bx)}^2 \right] 
    + \mbE \left[ \norm{ \gr f_i(\bx) - \gr L(\bx)}^2 \right]  \nonumber \\
    & \leq 
    \gamma^2 \left( \frac{\mbE[1/V]}{T_d} + \frac{1}{m} \right)  \nonumber \\
    & \leq
    2 \gamma^2 \max{\left\{\frac{\mbE[1/V]}{T_d},  \frac{1}{m}\right\}}. \nonumber 
\end{align}
\end{proof}

\subsection{Proof of Theorem \ref{thm1}}\label{sec:cnvx-proof}

To prove Theorem \ref{thm1}, we first establish two Lemmas \ref{lemma1} and \ref{lemma:convex2} and then easily conclude the theorem from the two results.

The main problem is to minimize the global objective defined in (\ref{global_cost}). We introduce the following optimization problem which is equivalent the main problem:
\begin{equation}\label{eq:main_2}
    \begin{aligned}
        \min_{\bx_1,\dots,\bx_n \in \mathbb{R}^p} \quad & F(\bx) \coloneqq \frac{1}{n} \sum_{i=1}^{n} f_i(\bx_i) \\
        \textrm{s.t.} \quad & \bx_i=\bx_j, \qquad  \text{for all}\ i, \ j\in \mathcal{N}_i,
    \end{aligned}
\end{equation}
where the vecor $\bx = [\bx_1;\cdots;\bx_n] \in \reals^{np}$ denotes the concatenation of all the local models. Clearly, $\tildbx^* \coloneqq [\bx^*; \cdots; \bx^*]$ is the solution to \eqref{eq:main_2}.  Using Assumption \ref{assump-W}, the constraint in the alternative problem (\ref{eq:main_2}) can be stated as $(\bI - \bW)^{1/2} \bx = 0$. Inspired by this fact, we define the following penalty function for every $\al$:
\begin{equation} \label{eq:h}
    h_{\al}(\bx) = \frac{1}{2} \bx^\top (\bI - \bW) \bx + \al  n F(\bx),
\end{equation}
and denote by $\bx^*_{\al}$ the (unique) minimizer of $h_{\al}(\bx)$. That is,
\begin{equation} \label{eq:hmin}
    \bx^*_{\al} = \operatorname*{\text{arg}\,\text{min}}_{\bx \in \mathbb{R}^{np}} h_{\al}(\bx) = \operatorname*{\text{arg}\,\text{min}}_{\bx \in \mathbb{R}^{np}} \frac{1}{2} \bx^\top \big(\bI - \bW \big) \bx + \al  n F(\bx).
\end{equation}

Next lemma characterizes the deviation of the models generated by the \texttt{QuanTimed-DSGD} method at iteration $T$, that is $\bx_T=[\bx_{1,T}; \cdots; \bx_{n,T}]$ from the optimizer of the penalty function, i.e. $\bx^*_{\al}$.

\begin{lemma} \label{lemma1}
Suppose Assumptions \ref{assump-W}--\ref{assump-convex} hold. Then, the expected deviation of the output of \texttt{QuanTimed-DSGD} from the solution to Problem \eqref{eq:h} is upper bounded by
\begin{equation}\label{eq:lemma1}
    \mbE \left[\norm{\bx_T - \bx^*_{\al}}^2 \right] 
    \leq
    \ccalO \left( \frac{n\sigma^2}{\mu}  \norm{W - W_D}^2 \right) \frac{1}{T^{\delta}} 
    +
    \ccalO \left( \frac{n\gamma^2}{\mu} \left( \frac{\mbE[1/V]}{T_d} + \frac{1}{m} \right) \right) \frac{1}{T^{2\delta}} ,
\end{equation}
for $\eps = T^{-3\delta/2}$, $\al = T^{-\delta/2}$, any $\delta \in (0,1/2)$ and $T \geq T^{\mathsf{c}}_{\mathsf{min1}}$, where 
\begin{equation}\label{eq:T-min-c}
    T^{\mathsf{c}}_{\mathsf{min1}} \coloneqq \text{max} \left\{ \left\lceil \left( \frac{(2 + K)^2}{\mu} \right)^{\frac{1}{\delta}} \right\rceil, \left\lceil e^{e^{\frac{1}{1-2\delta}}} \right\rceil,  \left\lceil  \mu^{\frac{1}{2\delta}} \right\rceil  \right\}.
\end{equation}
 \end{lemma}

\begin{proof}[Proof of Lemma \ref{lemma1}]
First note that the gradient of the penalty function $h_{\al}$ defined in \eqref{eq:h} is as follows:
\begin{equation}\label{eq:h-gr}
    {\gr}h_{\al}(\bx_t) = \left(\bI - \bW \right) \bx_t + \al n {\gr} F(\bx_{t}),
\end{equation}
where $\bx_t = [\bx_{1,t}; \cdots; \bx_{n,t}]$ denotes the concatenation of models at iteration $t$. Now consider the following stochastic gradient function for $h_{\al}$:
\begin{equation}\label{eq:htilde}
    \widetilde{\gr}h_{\al}(\bx_t) 
    =
    \left(\bW_{D} - \bW\right) \bz_{t} + \left(\bI - \bW_D \right) \bx_t + \al n \widetilde{\gr} F(\bx_{t}), 
\end{equation}
where 
\begin{align}
    \tNab F(\bx_t) &= \left[ \frac{1}{n} \tNab f_1(\bx_{1,t}); \cdots; \frac{1}{n} \tNab f_n(\bx_{n,t}) \right].  \nonumber 
\end{align}
We let $\cF^t$ denote a sigma algebra that measures the history of the system up until time $t$. According to Assumptions \ref{assump-Q} and \ref{assump-gr}, the stochastic gradient defined above is unbiased, that is,
\begin{align}
    \mbE \left[ \widetilde{\gr}h_{\al}(\bx_t) \vert \cF^t \right] 
    & =
    \left(\bW_{D} - \bW\right) \mbE \left[ \bz_{t} \vert \cF^t \right]
    +
    \left(\bI - \bW_D \right) \bx_t
    +
    \al n \mbE \left[ \widetilde{\gr} F(\bx_{t}) \vert \cF^t \right]  \nonumber \\
    & =
    \left(\bI - \bW \right) \bx_t + \al n {\gr} F(\bx_{t})  \nonumber \\
    & =
    \gr h_{\al}(\bx_t). \nonumber 
\end{align}
We can also write the update rule of \texttt{QuanTimed-DSGD} method as follows:
\begin{align} \label{eq:rule-matrix}
    \bx_{t+1} 
    & =
    \bx_{t} - \eps \left( \left(\bW_{D} - \bW \right) \bz_{t} + \left(\bI - \bW_D \right) \bx_t + \al n \widetilde{\gr} F(\bx_{t}) \right)  \nonumber \\
    & =
    \bx_{t} - \eps \widetilde{\gr} h_{\al}(\bx_t),
\end{align}
which also represents an iteration of the Stochastic Gradient Descent (SGD) algorithm with step-size $\eps$ in order to minimize the penalty function $h_{\al}(\bx)$ over $\bx \in \reals^{np}$. We can bound the deviation of the iteration generated by \texttt{QuanTimed-DSGD} from the optimizer $\bx^*_{\al}$ as follows:
\begin{align} 
    \mbE \left[\norm{\bx_{t+1} - \bx^*_{\al}}^2 | \cF^t \right] &= \mbE \left[\norm{\bx_{t} - \eps \widetilde{\gr}h_{\al}(\bx_t) - \bx^*_{\al}}^2 | \cF^t \right]  \nonumber \\
    & =
    \norm{\bx_{t} - \bx^*_{\al}}^2 
    - 2 \eps \left \langle \bx_{t} - \bx^*_{\al}, \mbE \left[ \widetilde{\gr}h_{\al}(\bx_t)| \cF^t \right] \right\rangle  \nonumber \\
    & \quad +
    \eps^2 \mbE \left[\norm{\widetilde{\gr}h_{\al}(\bx_t)}^2 | \cF^t \right]  \nonumber  \\
    & =
    \norm{\bx_{t} - \bx^*_{\al}}^2 
    - 2 \eps \left \langle \bx_{t} - \bx^*_{\al},  {\gr}h_{\al}(\bx_t) \right\rangle  \nonumber \\
    & \quad +
    \eps^2 \mbE \left[\norm{\widetilde{\gr}h_{\al}(\bx_t)}^2 | \cF^t \right]  \nonumber  \\
    & \leq \left(1-2 \mu_{\al} \eps \right)\norm{\bx_{t} - \bx^*_{\al}}^2 
    + 
    \eps^2 \mbE \left[\norm{\widetilde{\gr}h_{\al}(\bx_t)}^2 | \cF^t \right], \label{eq:bound-h-sgd}
\end{align}
where we used the fact that the penalty function $h_{\al}$ is strongly convex with parameter $\mu_{\al} \coloneqq \al \mu$. Moreover, we can bound the second term in RHS of \eqref{eq:bound-h-sgd} as follows:
\begin{align} 
    & \quad 
    \mbE  \left[ \norm{\widetilde{\gr}h_{\al}(\bx_t)}^2 \vert \cF^t \right]  \nonumber \\
    & = 
    \mbE \left[ \norm{ \left(\bW_{D} - \bW \right) \bz_{t} + \left(\bI - \bW_D \right) \bx_t + \al n \widetilde{\gr} F(\bx_{t})}^2 \vert \cF^t \right]  \nonumber \\
    & = 
    \mbE \left[ \norm{ \left( \bI - \bW \right) \bx_{t} + \al n \gr F(\bx_{t}) + \left(\bW_D - \bW \right) ( \bz_{t} -\bx_t) + \al n \widetilde{\gr} F(\bx_{t}) -  \al n \gr F(\bx_{t}) }^2 \vert \cF^t \right]  \nonumber \\
    & = \norm{{\gr}h_{\al}(\bx_t)}^2 + \mbE \left[ \norm{ \left(\bW_D - \bW \right) ( \bz_{t} -\bx_t)}^2  \vert \cF^t \right] + \al^2 n^2 \mbE \left[ \norm{ \widetilde{\gr} F(\bx_t) - \gr F(\bx_t)}^2 \vert \cF^t \right]   \nonumber \\
    &\leq K^2_{\al} \norm{\bx_{t} - \bx^*_{\al}}^2 + n \sigma^2 \norm{W-W_D}^2 + \al^2 n \gamma_2^2. \label{eq:bound-h-sg}
\end{align}
To derive \eqref{eq:bound-h-sg}, we used the facts that $h_{\al}$ is smooth with parameter $K_{\al} \coloneqq 1 - \lambda_n(W) + \al K$; the quantizer is unbiased with variance $\leq \sigma^2$ (Assumption \ref{assump-Q}); stochastic gradients of the loss function are unbiased and variance-bounded (Assumption \ref{assump-gr} and Lemma \ref{lemma:stoch-bound}). Plugging \eqref{eq:bound-h-sg} in \eqref{eq:bound-h-sgd} yields
\begin{align}
    \mbE \left[\norm{\bx_{t+1} - \bx^*_{\al}}^2 | \cF^t\right] & \leq
    \left(1-2 \mu_{\al} \eps 
    +
    \eps^2 K^2_{\al} \right)\norm{\bx_{t} - \bx^*_{\al}}^2
    +
    \eps^2 { n}\sigma^2 \norm{W-W_D}^2
    +
    \al^2 \eps^2 n \gamma_2^2. \label{eq:bound3}
\end{align}
To ease the notation, let $e_t := \mbE [ \, \norm{\bx_{t} - \bx^*_{\al}}^2 ]$ denote the expected  deviation of the models at iteration $t$ i.e. $\bx_{t}$ from the optimizer $\bx^*_{\al}$ with respect to all the randomnesses from iteration $t=0$. Therefore,
\begin{align}
    e_{t+1} 
    & \leq
    \left(1-2 \mu_{\al} \eps + \eps^2 K^2_{\al} \right)e_t 
    +
    \eps^2  n\sigma^2 \norm{W-W_D}^2
    +
    \al^2 \eps^2 \gamma_2^2  \nonumber \\
    & = 
    \left(1-\eps (2 \mu_{\al}  - \eps K^2_{\al}) \right)e_t +
    \eps^2 {n} \sigma^2 \norm{W-W_D}^2
    +
    \al^2 \eps^2 n \gamma_2^2. \label{eq:e_t+1-bound}
\end{align}
For any $T \geq T^{\mathsf{c}}_{\mathsf{min1}}$ and the proposed pick $\eps = T^{-3 \delta/2}$, we have
\begin{align}
    T^{\delta} \geq \left( T^{\mathsf{c}}_{\mathsf{min1}} \right)^{\delta} \geq  \frac{(2 + K)^2}{\mu}, \nonumber 
\end{align}
and therefore
\begin{align}
    \eps & = 
    \frac{1}{T^{3\delta/2}}\nonumber\\
    & \leq
    \frac{\mu }{(2 + K)^2} \cdot \frac{1}{T^{\delta/2}}\nonumber\\
    & \leq
    \frac{\mu_{\al}}{(2 + \al K )^2 }\nonumber\\
    & \leq
    \frac{\mu_{\al}}{K^2_{\al}}. \nonumber 
\end{align}
Hence, we can further bound \eqref{eq:e_t+1-bound} as follows:
\begin{align}
    e_{t+1} 
    & \leq
    \left( 1 - \eps \left(2 \mu_{\al}  - \eps K^2_{\al} \right) \right)e_t + \eps^2 {  n}\sigma^2 \norm{W-W_D}^2 + \al^2 \eps^2 n \gamma_2^2  \nonumber \\
    &\leq \left(1- \mu_{\al} \eps \right)e_t + \eps^2 {  n}\sigma^2\norm{W-W_D}^2 +
    \al^2 \eps^2 n \gamma_2^2 \nonumber \\
    &= \left(1 - \frac{\mu}{T^{2\delta}} \right) e_t + \frac{n \sigma^2 \norm{W-W_D}^2}{T^{3\delta}} + \frac{n \gamma_2^2}{T^{4\delta}}. \nonumber 
\end{align}
Now, we let $(a,b,c) = (\mu, n \sigma^2 \norm{W-W_D}^2, n \gamma_2^2)$ and employ Lemma \ref{lemma:e_t-itr} which yields
\begin{align}
    e_T
    & =
    \mbE \left[ \norm{\bx_T - \bx^*_{\al}}^2 \right]  \nonumber\\
    & \leq
    \ccalO \left(\frac{b/a}{T^{\delta}}\right)
    +
    \ccalO \left(\frac{c/a}{T^{2\delta}}\right)  \nonumber \\
    & =
    \ccalO \left( \frac{n\sigma^2}{\mu}  \norm{W - W_D}^2 \frac{1}{T^{\delta}} \right)
    +
    \ccalO \left( \frac{n\gamma^2}{\mu} \left( \frac{\mbE[1/V]}{T_d} + \frac{1}{m} \right) \frac{1}{T^{2\delta}} \right), \nonumber 
\end{align}
and the proof of Lemma \ref{lemma1} is concluded.
\end{proof}

Now we also bound the deviation of the optimizers of the penalty function and the main loss function, that is $\bx^*_{\al}$ and $\tildbx^*.$

\begin{lemma}\label{lemma:convex2}
Suppose Assumptions \ref{assump-W}, \ref{assump-smooth}--\ref{assump-convex} hold. Then the difference between the optimal solutions to \eqref{eq:main_2} and its penalized version \eqref{eq:hmin} is bounded above by
\begin{equation}  
\norm{\bx^*_{\al} - \tildbx^*} \leq  \cO \left( \frac{\sqrt{2n} c_2 D \left(3+ 2K/\mu\right)}{ 1-\beta}  \frac{1}{T^{\delta/2}} \right), \nonumber 
\end{equation}
for $\al = T^{-\delta/2}$, any $\delta \in (0,1/2)$ and $T \geq T^{\mathsf{c}}_{\mathsf{min2}}$ where
\begin{equation} 
    T^{\mathsf{c}}_{\mathsf{min2}}
    \coloneqq 
    \text{max} \left\{ \left\lceil \left(\frac{K}{1+\lambda_n(W)}\right)^{\frac{2}{\delta}} \right\rceil , \left\lceil (\mu + K)^{\frac{2}{\delta}} \right\rceil  \right\}. \nonumber 
\end{equation}
\end{lemma}

\begin{proof}[Proof of Lemma \ref{lemma:convex2}]
First, recall the penalty function minimization in (\ref{eq:hmin}). Following sequence is the update rule associated with this problem when the gradient descent method is applied to the objective function $h_{\al}$ with the unit step-size $\eta = 1$,
\begin{equation}\label{eq:gd}
\bu_{t+1} = \bu_{t} - \eta \gr h_{\al} (\bu_t)= \bW \bu_t - \al n \gr F(\bu_t).
\end{equation}
From analysis of GD for strongly convex objectives, the sequence $\{ \bu_{t} : t=0,1,\cdots\}$ defined above exponentially converges to the minimizer of $h_{\al}$, $\bx^*_{\al}$, provided that $1 = \eta \leq {2}/{K_{\al}}$. The latter condition is satisfied if we make $\al \leq (1 + \lambda_n(W))/{K}$. Therefore, 
\begin{align}
    \norm{\bu_{t} - \bx^*_{\al}}^2 &\leq (1- \mu_{\al})^{t}\norm{\bu_{0} - \bx^*_{\al}}^2  \nonumber \\
    & =
    (1- \al \mu)^{t}\norm{\bu_{0} - \bx^*_{\al}}^2. \nonumber 
\end{align}
If we take $\bu_{0}=0$, then (\ref{eq:gd}) implies 
\begin{align}
    \norm{\bu_{T} - \bx^*_{\al}}^2 &\leq   (1- \al \mu)^{T}  \norm{\bx^*_{\al}}^2\nonumber \\
    &\leq 2  (1- \al \mu)^{T}  \left( \norm{\tildbx^* - \bx^*_{\al}}^2 + \norm{\tildbx^*}^2 \right) \nonumber\\
    &= 2  (1- \al \mu)^{T}  \left( \norm{\tildbx^* - \bx^*_{\al}}^2 + n \norm{{\bx}^*}^2 \right) \label{eq:bounduT},
\end{align}
On the other hand, it can be shown (\citet{yuan2016convergence}) that if $\al \leq \text{min} \{ (1+\lambda_n(W))/K , 1/(\mu  + K) \}$, then the sequence $\{ \bu_{t} : t=0,1,\cdots\}$ defined in (\ref{eq:gd}) converges to the $\cO (\frac{\al}{1-\beta})$-neighborhood of the optima $\tildbx^*$, i.e.,
\begin{equation}\label{eq:gd3}
    \norm{\bu_{t} - \tildbx^*} \leq \cO\left(\frac{\al}{1-\beta} \right).
\end{equation}
If we take $\al=T^{-\delta/2}$, the condition $T \geq T_{\textsf{min-c2}}$ implies that $\al \leq \text{min} \{ (1+\lambda_n(W))/K , 1/(\mu  + K) \}$. Therefore, (\ref{eq:gd3}) yields
\begin{equation}\label{eq:gd4}
    \norm{\bu_{T} - \tildbx^*} \leq \cO \left(\frac{\al}{1-\beta} \right).
\end{equation}
More precisely, we have the following (See Corollary 9 in \citet{yuan2016convergence}):
\begin{equation} \label{eq:exactb}
    \norm{\bu_{T} - \tildbx^*} \leq \sqrt{n} \left(  c^T_3 \norm{{\bx}^*} + \frac{c_4}{\sqrt{1-c^2_3}} + \frac{\al D}{1-\beta} \right),
\end{equation}
where
\begin{equation}
    c^2_3 = 1 - \frac{1}{2} \cdot \frac{\mu K}{\mu + K}  \al , \nonumber 
\end{equation}
\begin{align}
	\frac{c_4}{\sqrt{1-c^2_3}} &= \frac{\al K D}{1-\beta} \sqrt{ 4 \left( \frac{\mu+K}{\mu K} \right)^2 - 2 \cdot \frac{\mu+K}{\mu K} \al } \nonumber \\
	& \leq \frac{2 \al  D }{  (1-\beta)} \left(1+K / \mu \right). \nonumber 
\end{align}
\begin{align}
	D^2 = 2 K  \sum_{i=1}^{n} \left( f_i(0) - f^*_i \right), \quad f^*_i =  \min_{\bx \in \mathbb{R}^p} f_i(\bx). \nonumber 
\end{align}
From (\ref{eq:exactb}) and (\ref{eq:gd4}), we have for $T \geq T_2$
\begin{align} \label{eq:boundT2}
   \norm{\bx^*_{\al} - \tildbx^*}^2 &= \norm{\bx^*_{\al} -\bu_{T}+\bu_{T}- \tildbx^*}^2 \nonumber\\
   &\leq 2\norm{\bx^*_{\al} -\bu_{T}}^2 + 2\norm{\bu_{T}- \tildbx^*}^2 \nonumber\\
   &\leq 4  (1- \al \mu)^{T}  \left( \norm{\tildbx^* - \bx^*_{\al}}^2 + n \norm{{\bx}^*}^2 \right) \nonumber\\
   & \quad + 2n \left(  \left( 1 - \frac{1}{2} \cdot \frac{\mu K}{\mu + K}  \al \right)^{T/2} \norm{{\bx}^*}  +  \frac{\al D}{1-\beta} \left(3+ 2K / \mu \right) \right)^2.
\end{align}
Note that for our pick $\al = T^{-\delta/2}$, we can write
\begin{align} 
	(1- \al \mu)^{T} \leq \exp\left( - T^{1-\delta/2} \right) & \eqqcolon e_1(T),  \nonumber\\
	\left( 1 - \frac{1}{2} \cdot \frac{\mu K}{\mu + K}  \al \right)^{T/2} \leq \exp \left( -\frac{1}{2} \cdot \frac{\mu K}{\mu + K} T^{1-\delta/2} \right) & \eqqcolon e_2(T). \nonumber 
\end{align}
Therefore, from (\ref{eq:boundT2}) we have
\begin{align} 
    & \quad \norm{\bx^*_{\al} - \tildbx^*}^2  \nonumber \\ 
    & \leq
    \frac{1}{\left( 1 - 4e_1(T)  \right)} 
    \Bigg\{ 4e_1(T) n \norm{{\bx}^*}^2  
    +
    2n e^2_2(T) \norm{{\bx}^*}^2
    +
    4n e_2(T) \norm{{\bx}^*} \frac{\al D}{1-\beta} \left(3+ 2K / \mu \right) \nonumber  \\
    & \quad +
    2n D^2 \left(3+ 2K / \mu \right)^2 \left( \frac{\al}{1-\beta}  \right)^2 \Bigg\}  \nonumber  \\
    &\leq  \frac{4n \left( 2 e_1(T) + e^2_2(T) \right)}{\left( 1 - 4e_1(T)  \right)} \frac{f_0 - f^*}{\mu} + \frac{4 \sqrt{2} n e_2(T)}{\left( 1 - 4e_1(T)  \right)} \sqrt{ \frac{f_0 - f^*}{\mu} } \frac{\al D}{1-\beta} \left(3+ 2K / \mu \right)  \nonumber \\
    & \quad +
    \frac{2n D^2 \left(3+ 2K / \mu \right)^2}{\left( 1 - 4e_1(T)  \right)} \left( \frac{\al}{1-\beta}  \right)^2,
\end{align}
where we used the fact that $\norm{{\bx}^*}^2 \leq 2(f_0 - f^*)/ \mu$ for $f_0 = f(0)$ and $f^* = \text{min}_{\bx \in \mathbb{R}^p} f(\bx) = f({\bx}^*)$. 
Given the fact that the terms $e_1(T)$ and $e_2(T)$ decay exponentially, i.e. $e_1(T)=o\left( \al^2 \right)$ and $e_2(T)=o\left( \al^2 \right)$,  we have 
\begin{align}
    \norm{\bx^*_{\al} - \tildbx^*} 
    &\leq 
    \cO \left( \sqrt{2n} D \left(3+ 2K/\mu\right)  \frac{\al}{1-\beta} \right) \nonumber \\
    &= 
    \cO \left( \frac{\sqrt{2n} D \left(3+ 2K/\mu\right)}{ 1-\beta}  \frac{1}{T^{\delta/2}} \right), \nonumber 
\end{align}
which concludes the claim in Lemma \ref{lemma:convex2}.
\end{proof}

Having proved Lemmas \ref{lemma1} and \ref{lemma:convex2}, we can now plug them in Theorem \ref{thm1} and write for $T \geq T^{\mathsf{c}}_{\mathsf{min}} \coloneqq \max\{ T^{\mathsf{c}}_{\mathsf{min1}}, T^{\mathsf{c}}_{\mathsf{min2}} \}$
\begin{align}
    \frac{1}{n} \sum_{i=1}^{n} \mbE \left[ \norm{\bx_{i,T} \!-\! {\bx}^*}^2 \right]
    & =
    \frac{1}{n} \mbE \left[\norm{\bx_{T} - \tildbx^*}^2\right]  \nonumber \\
    &= 
     \frac{1}{n}\mbE \left[\norm{\bx_T - \bx^*_{\al} + \bx^*_{\al} - \tildbx^*}^2\right] \nonumber  \nonumber \\
    &\leq
    \frac{2}{n} \mbE \left[\norm{\bx_T - \bx^*_{\al}}^2 \right] 
    +
    \frac{2}{n} \norm{\bx^*_{\al} - \tildbx^*}^2 \nonumber\\
    &\leq
    \ccalO\! \left( \frac{D^2 (K/\mu)^2 }{(1-\beta)^2} + \frac{\sigma^2}{\mu} \right) \!\frac{1}{T^{\delta}}  
    +
    \ccalO \!\left( \frac{\gamma^2}{\mu}  \max{\left\{\frac{\mbE[1/V]}{T_d},  \frac{1}{m}\right\}} \right)\! \frac{1}{T^{2\delta}}. \nonumber 
\end{align}

In the end, we state and proof Lemma \ref{lemma:e_t-itr} which we used its result earlier in the proof of Lemma \ref{lemma1}.

\begin{lemma} \label{lemma:e_t-itr}
Let the non-negative sequence $e_t$ satisfy the inequality
\begin{equation}\label{eq:recursive_expression}
    e_{t+1} 
    \leq
    \left(1-\frac{a}{T^{2\delta}}\right) e_t 
    +
    \frac{b}{T^{3\delta}}
    +
    \frac{c}{T^{3\delta}},
\end{equation}
for $t=0,1,2,\cdots$, positive constants $a,b,c$ and $\delta \in (0,1/2)$. Then, after 
\begin{equation}
    T \geq \max \left\{  \left\lceil e^{e^{\frac{1}{1-2\delta}}} \right\rceil,  \left\lceil  a^{\frac{1}{2\delta}} \right\rceil \right\} \nonumber 
\end{equation}
iterations, the iterate $e_T$ satisfies 
\begin{equation}\label{eq:e_T-convergence}
    e_T 
    \leq  
    \ccalO \left(\frac{b/a}{T^{\delta}}\right)
    +
    \ccalO \left(\frac{c/a}{T^{2\delta}}\right).
\end{equation}
\end{lemma}

\begin{proof}[Proof of Lemma \ref{lemma:e_t-itr}]
Use the expression in \eqref{eq:recursive_expression} for steps $t-1$ and $t$ to obtain
\begin{align}
    e_{t+1} 
    & \leq  
    \left(1-\frac{a}{T^{2\delta}}\right)^2 e_{t-1} 
    +
    \left[1 + \left(1-\frac{a}{T^{2\delta}}\right) \right]\frac{b}{T^{3\delta}}
    +
    \left[1 + \left(1-\frac{a}{T^{2\delta}}\right) \right]\frac{c}{T^{4\delta}}, \nonumber 
\end{align}
where $T \geq a^{1/(2\delta)}$. By recursively applying these inequalities for all steps $t = 0,1,\cdots$ we obtain that  
\begin{align}
    e_{t} & \leq
    \left(1-\frac{a}{T^{2\delta}}\right)^t e_{0}  \nonumber \\
    & \quad +
    \frac{b}{T^{3\delta}} \left[1+\left(1-\frac{a}{T^{2\delta}}\right) +\dots +\left(1-\frac{a}{T^{2\delta}}\right)^{t-1} \right]  \nonumber \\
    & \quad +
    \frac{c}{T^{4\delta}} \left[1+\left(1-\frac{a}{T^{2\delta}}\right) +\dots +\left(1-\frac{a}{T^{2\delta}}\right)^{t-1} \right]  \nonumber \\
    & \leq
    \left(1-\frac{a}{T^{2\delta}}\right)^t e_{0}
    +
    \frac{b}{T^{3\delta}} \left[\sum_{s=0}^{t-1}\left(1-\frac{a}{T^{2\delta}}\right)^{s} \right]
    +
    \frac{c}{T^{4\delta}} \left[\sum_{s=0}^{t-1}\left(1-\frac{a}{T^{2\delta}}\right)^{s} \right]  \nonumber \\
    & \leq
    \left(1-\frac{a}{T^{2\delta}}\right)^t e_{0} 
    +
    \frac{b}{T^{3\delta}} \left[\sum_{s=0}^{\infty}\left(1-\frac{a}{T^{2\delta}}\right)^{s} \right] 
    +
    \frac{c}{T^{4\delta}} \left[\sum_{s=0}^{\infty}\left(1-\frac{a}{T^{2\delta}}\right)^{s} \right]  \nonumber \\
    & =  
    \left(1-\frac{a}{T^{2\delta}}\right)^t e_{0} 
    +
    \frac{b}{T^{3\delta}} \left[\frac{1}{1- \left( 1-\frac{a}{T^{2\delta}} \right) } \right] 
    +
    \frac{c}{T^{4\delta}} \left[\frac{1}{1- \left( 1-\frac{a}{T^{2\delta}} \right) } \right] \nonumber  \\
    & = 
    \left(1-\frac{a}{T^{2\delta}}\right)^t e_{0} 
    +
    \frac{b/a}{T^{\delta}}
    +
    \frac{c/a}{T^{2\delta}}.   \nonumber 
\end{align}
Therefore, for the iterate corresponding to step $t=T$ we can write 
\begin{align}
    e_{T} 
    & \leq 
    \left(1-\frac{a}{T^{2\delta}}\right)^T e_{0} 
    +
    \frac{b/a}{T^{\delta}}
    +
    \frac{c/a}{T^{2\delta}}  \nonumber \\
    & \leq 
    \exp \left({-aT^{(1-2\delta)}} \right) e_{0} 
    +
    \frac{b/a}{T^{\delta}}
    +
    \frac{c/a}{T^{2\delta}}  \label{eq:withexp} \\
    & =
    \ccalO \left(\frac{b/a}{T^{\delta}}\right)
    +
    \ccalO \left(\frac{c/a}{T^{2\delta}}\right), \nonumber 
\end{align}
and the claim in \eqref{eq:e_T-convergence} follows. Note that for the last inequality we assumed that the exponential term in is negligible comparing to the sublinear term. It can be verified for instance if $1-2\delta$ is of $\mathcal{O} \left( 1/\log(\log(T)) \right)$ or greater than that, it satisfies this condition. Moreover, setting $\delta=1/2$ results in a constant (and hence non-vanishing) term in (\ref{eq:withexp}). 
\end{proof}

\subsection{Proof of Theorem \ref{thm2}}\label{sec:noncnvx-proof}

To ease the notation, we agree in this section on the following shorthand notations for $t=0,1,2,\cdots$:
\begin{align}
    X_t 
    &= 
    [\bx_{1,t} \,\, \cdots \,\, \bx_{n,t}] \in \reals^{p \times n},  \nonumber \\
    Z_t
    &=
    [\bz_{1,t} \,\, \cdots \,\, \bz_{n,t}] \in \reals^{p \times n},  \nonumber \\
    \barbx_t 
    &=
    \frac{1}{n} \sum_{i=1}^{n} \bx_{i,t}  \in \reals^{p},  \nonumber \\
    \bbarX_t
    & =
    [\barbx_t \,\, \cdots \,\, \barbx_t] \in \reals^{p \times n},  \nonumber \\
    \tparf (X_t) &= \left[ \tNab f_1(\bx_{1,t}) \,\, \cdots \,\, \tNab f_n(\bx_{n,t}) \right] \in \reals^{p \times n},  \nonumber \\
    \parf (X_t) &= \left[ \gr f_1(\bx_{1,t}) \,\, \cdots \,\, \gr f_n(\bx_{n,t}) \right] \in \reals^{p \times n}. \nonumber 
\end{align}

As stated before, we can write the update rule of the proposed \texttt{QuanTimed-DSGD} in the following matrix form:
\begin{equation}\label{eq:update_matrix_1}
    X_{t+1} =  X_t \left( (1 - \eps)I + \eps W \right) + \eps (Z_t - X_t) (W - W_D) - \al \eps \tparf(X_t).
\end{equation}
Let us denote $W_{\eps} = (1 - \eps)I + \eps W$ and write (\ref{eq:update_matrix_1}) as
\begin{equation}\label{eq:update_matrix_2}
    X_{t+1} =  X_t W_{\eps} + \eps (Z_t - X_t) (W - W_D) - \al \eps \tparf(X_t).
\end{equation}
Clearly for any $\eps \in (0,1]$, $W_{\eps}$ is also doubly stochastic with eigenvalues $\lambda_i(W_{\eps}) = 1 - \eps + \eps \lambda_i(W)$ and spectral gap $1 - \beta_{\eps} = 1 - \max \left\{|\lambda_2(W_{\eps})|,|\lambda_n(W_{\eps})| \right\}$.

We start the convergence analysis by using the smoothness property of the objectives and write
\begin{align}
    \mbE f \left( \frac{X_{t+1} \bone_n}{n} \right) 
    & = 
    \mbE f \left( \frac{X_t  W_{\eps} \bone_n}{n} + \frac{\eps (Z_t - X_t) (W - W_D) \bone_n}{n} - \frac{\al \eps \tparf(X_t) \bone_n}{n} \right)  \nonumber \\
    & \overset{\text{Assumption \ref{assump-smooth}}}\leq 
    \mbE f \left( \frac{X_{t} \bone_n}{n} \right) - \al \eps \mbE \left \langle \gr f \left( \frac{X_{t} \bone_n}{n} \right) , \frac{ \parf(X_{t}) \bone_n}{n} \right \rangle \nonumber\\
    & \quad + \frac{\eps^2K}{2} \mbE \norm{ \frac{(Z_t - X_t) (W - W_D) \bone_n}{n} - \al \frac{ \tparf(X_{t}) \bone_n}{n} }^2. \label{eq:noncnvx-1}
\end{align}
We specifically used the following equivalent form of the smoothness (Assumption \ref{assump-smooth}) for every local and hence the global objective
\begin{equation}
    f_i(\bby) \leq f_i(\mathbf{x}) + \left \langle \gr f_i (\mathbf{x}) , \bby - \mathbf{x} \right \rangle + \frac{K}{2} \norm{\bby - \mathbf{x}}^2, \quad \text{ for all } i \in [n], \bx, \bby \in \reals^p. \nonumber 
\end{equation}
Also, we used the following simple fact: 
\begin{equation}
    W_{\eps} \bone_n = ((1 - \eps)I + \eps W) \bone_n = (1 - \eps)\bone_n + \eps W \bone_n = \bone_n  \nonumber 
\end{equation}
Now let us bound the term in (\ref{eq:noncnvx-1}) as follows:
\begin{align}
    \mbE \norm{ \frac{(Z_t - X_t) (W - W_D) \bone_n}{n} - \al \frac{ \tparf(X_{t}) \bone_n}{n} }^2 
    & =
    \mbE \norm{ \frac{(Z_t - X_t) (W - W_D) \bone_n}{n}}^2  \nonumber \\
    & \quad + \mbE \norm{\al \frac{ \tparf(X_{t}) \bone_n}{n} }^2  \nonumber \\
    & =
    \frac{1}{n^2} \sum_{i=1}^{n} (1 - w_{ii})^2 \mbE \norm{ \bz_{i,t} - \bx_{i,t}}^2  \nonumber \\
    & \quad +
    \al^2 \mbE \norm{\frac{ \tparf(X_{t}) \bone_n}{n} }^2  \nonumber \\
    & \leq 
    \frac{\sigma^2}{n} + \al^2 \mbE \norm{\frac{ \tparf(X_{t}) \bone_n}{n} }^2, \label{eq:noncnvx-2}
\end{align}
where we used Assumption \ref{assump-Q} to derive the first term in (\ref{eq:noncnvx-2}). To bound the second term in (\ref{eq:noncnvx-2}), we have
\begin{align}
    \mbE \norm{\frac{ \tparf(X_{t}) \bone_n}{n} }^2 
    & = 
    \mbE \norm{ \frac{ \sum_{i=1}^{n} \tNab f_i (\bx_{i,t}) }{n} }^2  \nonumber \\
    & = 
    \mbE \norm{ \frac{ \sum_{i=1}^{n} \tNab f_i (\bx_{i,t}) - \gr f_i (\bx_{i,t}) + \gr f_i (\bx_{i,t}) }{n} }^2 \nonumber  \\
    & = 
    \mbE \norm{ \frac{ \sum_{i=1}^{n} \tNab f_i (\bx_{i,t}) - \gr f_i (\bx_{i,t}) }{n} }^2
    + \mbE \norm{ \frac{ \sum_{i=1}^{n} \gr f_i (\bx_{i,t}) }{n} }^2  \nonumber \\
    & \leq 
    \frac{\gamma^2}{n} \left( \frac{\mbE[1/V]}{T_d} + \frac{1}{m} \right)
    +
    \mbE \norm{ \frac{ \sum_{i=1}^{n} \gr f_i (\bx_{i,t}) }{n} }^2  \nonumber \\
    & =
    \frac{\gamma_2^2}{n}
    +
    \mbE \norm{ \frac{ \sum_{i=1}^{n} \gr f_i (\bx_{i,t}) }{n} }^2. \label{eq:stoch-bound}
\end{align}
where the last inequality follows from Lemma \ref{lemma:stoch-bound}.

Plugging (\ref{eq:stoch-bound}) in (\ref{eq:noncnvx-1}) yields
\begin{align}
    \mbE f \left( \frac{X_{t+1} \bone_n}{n} \right)
    & \leq 
    \mbE f \left( \frac{X_{t} \bone_n}{n} \right) 
    -
    \al \eps \mbE \left \langle \gr f \left( \frac{X_{t} \bone_n}{n} \right) , \frac{ \parf(X_{t}) \bone_n}{n} \right \rangle  \nonumber \\
    & \quad + 
    \frac{\eps^2  K }{2n} \sigma^2
    +
    \frac{\al^2 \eps^2 K }{2 n} \gamma_2^2 
    +
    \frac{\al^2 \eps^2 K }{2} \mbE \norm{ \frac{ \sum_{i=1}^{n} \gr f_i (\bx_{i,t}) }{n} }^2  \nonumber \\
    & =
    \mbE f \left( \frac{X_{t} \bone_n}{n} \right) - \frac{\al \eps - \al^2 \eps^2 K }{2} \mbE \norm{\frac{ \parf(X_{t}) \bone_n}{n} }^2 - \frac{\al \eps}{2} \mbE \norm{ \gr f \left( \frac{X_{t} \bone_n}{n} \right)}^2  \nonumber \\
    & \quad + 
    \frac{\eps^2  K }{2n} \sigma^2
    +
    \frac{\al^2 \eps^2 K }{2 n} \gamma_2^2  \nonumber \\
    & \quad + 
    \frac{\al \eps}{2} \underbrace{\mbE \norm{ \gr f \left( \frac{X_{t} \bone_n}{n} \right) - \frac{ \parf(X_{t}) \bone_n}{n}}^2}_\textrm{$T_1$} \label{eq:T1}
\end{align}
where we used the identity $2 \langle \mathbf{a}, \mathbf{b} \rangle = \norm{\mathbf{a}}^2 + \norm{\mathbf{b}}^2 - \norm{\mathbf{a} - \mathbf{b}}^2$. The term $T_1$ defined in (\ref{eq:T1}) can be bounded as follows:
\begin{align}
    T_1 
    & =
    \mbE \norm{ \gr f \left( \frac{X_{t} \bone_n}{n} \right) - \frac{ \parf(X_{t}) \bone_n}{n}}^2  \nonumber \\
    & \leq 
    \frac{1}{n} \sum_{i=1}^{n} \mbE \norm{ \gr f_i \left( \frac{X_{t} \bone_n}{n} \right) -  \gr f_i (\bx_{i,t})}^2  \nonumber \\
    & \leq \frac{ K ^2}{n} \sum_{i=1}^{n} \underbrace{ \mbE \norm{ \frac{X_{t} \bone_n}{n} - \bx_{i,t} }^2}_\textrm{$Q_{i,t}$}. \nonumber 
\end{align}
Let us define 
\begin{align}
    Q_{i,t} 
    \coloneqq
    \mbE \norm{ \frac{X_{t} \bone_n}{n} - \bx_{i,t} }^2, \nonumber 
\end{align}
and
\begin{align}
    M_t 
    \coloneqq
    \frac{1}{n} \sum_{i=1}^{n} Q_{i,t} 
    =
    \frac{1}{n} \sum_{i=1}^{n} \mbE \norm{ \frac{X_{t} \bone_n}{n} - \bx_{i,t} }^2. \nonumber 
\end{align}
Here, $Q_{i,t}$ captures the deviation of the model at node $i$ from the average model at iteration $t$ and $M_t$ aggregates them to measure the average total consensus error. To bound $M_t$, we need to evaluate the following recursive expressions:
\begin{align}
    X_{t}
    & =
    X_{t-1}  W_{\eps} + \eps (Z_{t-1} - X_{t-1}) (W - W_D) - \al \eps \tparf(X_{t-1}) \nonumber \\
    & =
    X_{0}  W^t_{\eps} + \eps \sum_{s=0}^{t-1} (Z_s - X_s) (W - W_D) W^{t-s-1}_{\eps}  - \al \eps \sum_{s=0}^{t-1} \tparf(X_{s}) W^{t-s-1}_{\eps}. \label{eq:rcrsv-2}
\end{align}
Now, using (\ref{eq:rcrsv-2}) we can write
\begin{align}
    M_t 
    & =
    \frac{1}{n} \sum_{i=1}^{n} \mbE \norm{ \frac{X_{t} \bone_n}{n} - \bx_{i,t} }^2  \nonumber \\
    & = 
    \frac{1}{n} \mbE \norm{ \bbarX_t - X_t }_F^2  \nonumber \\
    & =
    \frac{1}{n} \mbE \norm{ X_t \frac{\bone \bone^{\top}}{n} - X_t }_F^2 \nonumber  \\ 
    & =
    \frac{1}{n} \mbE \Bigg\Vert X_0 \left( \frac{\bone \bone^{\top}}{n} - W^t_{\eps} \right) 
    +
    \eps \sum_{s=0}^{t-1} (Z_s - X_s) (W - W_D) \left( \frac{\bone \bone^{\top}}{n}- W^{t-s-1}_{\eps} \right)  \nonumber \\
    & \quad \quad 
    -
    \al \eps \sum_{s=0}^{t-1} \tparf(X_{s}) \left( \frac{\bone \bone^{\top}}{n}- W^{t-s-1}_{\eps} \right) \Bigg\Vert_F^2  \nonumber \\
    & =
    \frac{( \al \eps )^2}{n} \underbrace{ \mbE \norm{ \sum_{s=0}^{t-1} \tparf(X_{s}) \left( \frac{\bone \bone^{\top}}{n}- W^{t-s-1}_{\eps} \right) }_F^2 }_\textrm{$T_2$} \nonumber \\
    & \quad + 
    \frac{\eps^2}{n} \underbrace{ \mbE \norm{ \sum_{s=0}^{t-1} (Z_s - X_s) (W - W_D) \left( \frac{\bone \bone^{\top}}{n}- W^{t-s-1}_{\eps} \right) }_F^2 }_\textrm{$T_3$}, \nonumber 
\end{align}
where we used the fact that quantiziations and stochastic gradients are statistically independent and $X_0 = 0$. We continue the analysis by bounding $T_2$ as follows:
\begin{align}
    T_2 
    & =
    \mbE \norm{ \sum_{s=0}^{t-1} \tparf(X_{s}) \left( \frac{\bone \bone^{\top}}{n}- W^{t-s-1}_{\eps} \right) }_F^2 \nonumber  \\
    & =
    \mbE \norm{ \sum_{s=0}^{t-1} \left( \tparf(X_{s}) - \parf(X_{s}) + \parf(X_{s})\right) \left( \frac{\bone \bone^{\top}}{n}- W^{t-s-1}_{\eps} \right) }_F^2  \nonumber \\
    & \leq
    2 \underbrace{ \mbE \norm{ \sum_{s=0}^{t-1} \left( \tparf(X_{s}) - \parf(X_{s}) \right) \left( \frac{\bone \bone^{\top}}{n}- W^{t-s-1}_{\eps} \right) }_F^2 }_\textrm{$T_4$} \nonumber \\
    & \quad + 
    2 \underbrace{ \mbE \norm{ \sum_{s=0}^{t-1} \parf(X_{s})  \left( \frac{\bone \bone^{\top}}{n}- W^{t-s-1}_{\eps} \right) }_F^2 }_\textrm{$T_5$}. \nonumber 
\end{align}
We can write
\begin{align}
    T_4 
    & =
    \mbE \norm{ \sum_{s=0}^{t-1} \left( \tparf(X_{s}) - \parf(X_{s}) \right) \left( \frac{\bone \bone^{\top}}{n}- W^{t-s-1}_{\eps} \right) }_F^2  \nonumber \\
    & \leq
    \sum_{s=0}^{t-1} \mbE \norm{ \tparf(X_{s}) - \parf(X_{s})}_F^2 \cdot \norm{ \frac{\bone \bone^{\top}}{n}- W^{t-s-1}_{\eps} }^2  \nonumber \\
    & \leq
    n \gamma_2^2 \sum_{s=0}^{t-1} \beta^{2(t-s-1)}_{\eps} \label{eq:T4-bound} \\
    & \leq
    \frac{n \gamma_2^2}{1 - \beta^{2}_{\eps}}, \nonumber 
\end{align}
where we used the facts that $\norm{AB}_F \leq \norm{A}_F \, \norm{B}$ for matrices $A,B$ and also that $\norm{ \frac{\bone \bone^{\top}}{n}- W^{t}_{\eps} } \leq \beta^{t}_{\eps}$ for any $t=0,1,\cdots$.
We continue by bounding $T_5$:
\begin{align}
    T_5 
    & =
    \mbE \norm{ \sum_{s=0}^{t-1} \parf(X_{s})  \left( \frac{\bone \bone^{\top}}{n}- W^{t-s-1}_{\eps} \right) }_F^2  \nonumber \\
    & = 
    \underbrace{ \sum_{s=0}^{t-1} \mbE \norm{ \parf(X_{s})  \left( \frac{\bone \bone^{\top}}{n}- W^{t-s-1}_{\eps} \right) }_F^2 }_\textrm{$T_6$}  \nonumber \\
    & \quad + 
    \underbrace{ \sum_{0 \leq s \neq s' \leq t-1} \mbE \left \langle \parf(X_{s})  \left( \frac{\bone \bone^{\top}}{n}- W^{t-s-1}_{\eps} \right) , \parf(X_{s'})  \left( \frac{\bone \bone^{\top}}{n}- W^{t-s'-1}_{\eps} \right) \right \rangle_F  }_\textrm{$T_7$} \label{eq:T_5}
\end{align}
Let us first bound the term $T_6$:
\begin{align}
    T_6 
    & =
    \sum_{s=0}^{t-1} \mbE \norm{ \parf(X_{s})  \left( \frac{\bone \bone^{\top}}{n}- W^{t-s-1}_{\eps} \right) }_F^2  \nonumber \\
    & \leq
    \sum_{s=0}^{t-1} \underbrace{ \mbE \norm{ \parf(X_{s}) }_F^2}_\textrm{$T_8$} \norm{ \frac{\bone \bone^{\top}}{n}- W^{t-s-1}_{\eps} }^2, \label{eq:T6-bound}
\end{align}
where
\begin{align}
    T_8
    & = 
    \mbE \norm{ \parf(X_{s}) }_F^2  \nonumber \\
    & \leq 
    3 \mbE \norm{ \parf(X_{s}) - \parf \left( \frac{X_{s} \bone_n}{n} \bone^{\top}_n \right)}_F^2  \nonumber \\
    & \quad + 
    3 \mbE \norm{ \parf \left( \frac{X_{s} \bone_n}{n} \bone^{\top}_n \right) - \gr f \left( \frac{X_{s} \bone_n}{n} \right)  \bone^{\top}_n}_F^2  \nonumber \\
    & \quad + 
    3 \mbE \norm{\gr f \left( \frac{X_{s} \bone_n}{n} \right)  \bone^{\top}_n}_F^2  \nonumber \\
    & \leq 
    3 \mbE \norm{ \parf(X_{s}) - \parf \left( \frac{X_{s} \bone_n}{n} \bone^{\top}_n \right)}_F^2  \nonumber \\
    & \quad + 
    3 n \gamma_1^2  \nonumber \\
    & \quad + 
    3 \mbE \norm{\gr f \left( \frac{X_{s} \bone_n}{n} \right)  \bone^{\top}_n}_F^2  \nonumber \\
    & \leq 
    3  K^2 \sum_{i=1}^{n}  \mbE  \norm{ \frac{X_{s} \bone_n}{n} - \bx_{i,s} }^2 + 3 n \gamma_1^2 + 3 \mbE \norm{\gr f \left( \frac{X_{s} \bone_n}{n} \right)  \bone^{\top}_n}_F^2  \nonumber \\
    & =
    3  K^2 \sum_{i=1}^{n}  Q_{i,s} + 3 n \gamma_1^2 + 3 \mbE \norm{\gr f \left( \frac{X_{s} \bone_n}{n} \right)  \bone^{\top}_n}_F^2. \label{eq:T8} 
\end{align}
Plugging \eqref{eq:T8} in (\ref{eq:T6-bound}) yields
\begin{align}
    T_6
    & \leq 
    3 K^2 \sum_{s=0}^{t-1} \sum_{i=1}^{n} Q_{i,s} \norm{ \frac{\bone \bone^{\top}}{n}- W^{t-s-1}_{\eps} }^2  \nonumber \\
    & \quad + 
    3 n \gamma_1^2 \frac{1}{1 - \beta^{2}_{\eps}} \nonumber  \\
    & \quad +
    3 \sum_{s=0}^{t-1} \mbE \norm{\gr f \left( \frac{X_{s} \bone_n}{n} \right)  \bone^{\top}_n}_F^2 \norm{ \frac{\bone \bone^{\top}}{n}- W^{t-s-1}_{\eps} }^2 \nonumber 
\end{align}
Going back to terms $T_5$ and $T_7$, we can write
\begin{align}
    T_7 
    & = 
    \sum_{s \neq s'}^{t-1} \mbE \left \langle \parf(X_{s})  \left( \frac{\bone \bone^{\top}}{n}- W^{t-s-1}_{\eps} \right) , \parf(X_{s'})  \left( \frac{\bone \bone^{\top}}{n}- W^{t-s'-1}_{\eps} \right) \right \rangle_F  \nonumber \\
    & \leq
    \sum_{s \neq s'}^{t-1} \mbE \norm{ \parf(X_{s})  \left( \frac{\bone \bone^{\top}}{n}- W^{t-s-1}_{\eps} \right)}_F \norm{ \parf(X_{s'})  \left( \frac{\bone \bone^{\top}}{n}- W^{t-s'-1}_{\eps} \right)}_F \nonumber \\
    & \leq
    \sum_{s \neq s'}^{t-1} \mbE \norm{ \parf(X_{s})}_F \norm{ \frac{\bone \bone^{\top}}{n}- W^{t-s-1}_{\eps}  } \norm{ \parf(X_{s'})}_F \norm{ \frac{\bone \bone^{\top}}{n}- W^{t-s'-1}_{\eps}  } \nonumber \\
    & \leq 
    \sum_{s \neq s'}^{t-1} \mbE \frac{\norm{ \parf(X_{s})}_F^2}{2} \norm{ \frac{\bone \bone^{\top}}{n}- W^{t-s-1}_{\eps}  } \norm{ \frac{\bone \bone^{\top}}{n}- W^{t-s'-1}_{\eps}  } \nonumber  \\
    & \quad + 
    \sum_{s \neq s'}^{t-1} \mbE \frac{\norm{ \parf(X_{s'})}_F^2}{2} \norm{ \frac{\bone \bone^{\top}}{n}- W^{t-s-1}_{\eps}  } \norm{ \frac{\bone \bone^{\top}}{n}- W^{t-s'-1}_{\eps}  } \nonumber  \\
    & \leq 
    \sum_{s \neq s'}^{t-1} \mbE \left( \frac{\norm{ \parf(X_{s})}_F^2}{2} + \frac{\norm{ \parf(X_{s'})}_F^2}{2} \right) \beta_{\eps}^{2t-(s+s')-2}  \nonumber \\
    & =
    \sum_{s \neq s'}^{t-1} \mbE \norm{ \parf(X_{s})}_F^2 \beta_{\eps}^{2t-(s+s')-2} \\
    & \leq 
    \underbrace{ 3 \sum_{s \neq s'}^{t-1} \left(3 K^2 \sum_{i=1}^{n} Q_{i,s} 
    +
    3 \mbE \norm{\gr f \left( \frac{X_{s} \bone_n}{n} \right)  \bone^{\top}_n}_F^2 \right) \beta_{\eps}^{2t-(s+s')-2} }_\textrm{$T_9$}  \nonumber \\
    & \quad +
    \underbrace{ 3 n  \gamma_1^2 \sum_{s \neq s'}^{t-1}  \beta_{\eps}^{2t-(s+s')-2} }_\textrm{$T_{10}$}. \nonumber 
\end{align}
In above, the term $T_{10}$ can be simply bounded as:
\begin{align*}
    T_{10} 
    & =
    3 n  \gamma_1^2 \sum_{s \neq s'}^{t-1}  \beta_{\eps}^{2t-(s+s')-2} \\
    & =
    6 n \gamma_1^2 \sum_{s > s'}^{t-1}   \beta_{\eps}^{2t-(s+s')-2} \\
    & =
    6 n \gamma_1^2 \frac{ \left( \beta_{\eps}^t - 1 \right) \left( \beta_{\eps}^t - \beta_{\eps} \right)}{ \left( \beta_{\eps} - 1 \right)^2 \left( \beta_{\eps} + 1 \right)} \\
    & \leq
    6 n \gamma_1^2 \frac{1}{\left( 1 -  \beta_{\eps} \right)^2}.
\end{align*}
The other term, i.e. $T_9$ can be bounded as follows:
\begin{align*}
    T_9 
    & =
    3 \sum_{s \neq s'}^{t-1} \left(3 K^2 \sum_{i=1}^{n} Q_{i,s} 
    +
    3 \mbE \norm{\gr f \left( \frac{X_{s} \bone_n}{n} \right)  \bone^{\top}_n}_F^2 \right) \beta_{\eps}^{2t-(s+s')-2} \\
    & =
    6 \sum_{s = 0}^{t-1} \left(3 K^2  \sum_{i=1}^{n}  Q_{i,s} + 3 \mbE \norm{\gr f \left( \frac{X_{s} \bone_n}{n} \right)  \bone^{\top}_n}_F^2 \right) \sum_{s' = s + 1}^{t-1} \beta_{\eps}^{2t-(s+s')-2} \\
    & \leq
    6 \sum_{s = 0}^{t-1} \left(3 K^2  \sum_{i=1}^{n} Q_{i,s} + 3 \mbE \norm{\gr f \left( \frac{X_{s} \bone_n}{n} \right)  \bone^{\top}_n}_F^2 \right) \frac{\beta_{\eps}^{t-s-1}}{ 1- \beta_{\eps}}
\end{align*}
Now that we have bounded $T_6$ and $T_7$, we go back and plug in \eqref{eq:T_5} to bound $T_5$:
\begin{align*}
    T_5
    & \leq
    3 K^2 \sum_{s=0}^{t-1} \sum_{i=1}^{n} Q_{i,s} \norm{ \frac{\bone \bone^{\top}}{n}- W^{t-s-1}_{\eps} }^2\\
    & \quad + 
    3 \sum_{s=0}^{t-1} \mbE \norm{\gr f \left( \frac{X_{s} \bone_n}{n} \right)  \bone^{\top}_n}_F^2 \norm{ \frac{\bone \bone^{\top}}{n}- W^{t-s-1}_{\eps} }^2 \\
    & \quad +
    6 \sum_{s = 0}^{t-1} \left(3 K^2  \sum_{i=1}^{n} Q_{i,s} + 3 \mbE \norm{\gr f \left( \frac{X_{s} \bone_n}{n} \right)  \bone^{\top}_n}_F^2 \right) \frac{\beta_{\eps}^{t-s-1}}{ 1- \beta_{\eps}} \\
    & \quad +
    3 n \gamma_1^2 \frac{1}{1 - \beta^{2}_{\eps}} \\
    & \quad +
    6 n \gamma_1^2 \frac{1}{\left( 1 -  \beta_{\eps} \right)^2} \\
    & \leq
    3 K^2 \sum_{s=0}^{t-1} \sum_{i=1}^{n} Q_{i,s} \norm{ \frac{\bone \bone^{\top}}{n}- W^{t-s-1}_{\eps} }^2\\
    & \quad + 
    3 \sum_{s=0}^{t-1} \mbE \norm{\gr f \left( \frac{X_{s} \bone_n}{n} \right)  \bone^{\top}_n}_F^2 \norm{ \frac{\bone \bone^{\top}}{n}- W^{t-s-1}_{\eps} }^2 \\
    & \quad +
    6 \sum_{s = 0}^{t-1} \left(3 K^2  \sum_{i=1}^{n} Q_{i,s} + 3 \mbE \norm{\gr f \left( \frac{X_{s} \bone_n}{n} \right)  \bone^{\top}_n}_F^2 \right) \frac{\beta_{\eps}^{t-s-1}}{ 1- \beta_{\eps}} \\
    & \quad +
    9 n \gamma_1^2 \frac{1}{(1 - \beta_{\eps})^{2}}, \\
\end{align*}
where we used the fact that $\frac{1}{1 - \beta_{\eps}^{2}} \leq \frac{1}{(1 - \beta_{\eps})^{2}}$. Now we bound the term $T_2$ having $T_4$ and $T_5$ bounded:
\begin{align*}
    T_2
    & =
    2 T_4 + 2 T_5 \\
    & \leq
    2 \frac{n \gamma_2^2}{1 - \beta^{2}_{\eps}} \\
    & \quad +
    6 K^2 \sum_{s=0}^{t-1} \sum_{i=1}^{n} Q_{i,s} \norm{ \frac{\bone \bone^{\top}}{n}- W^{t-s-1}_{\eps} }^2\\
    & \quad + 
    6 \sum_{s=0}^{t-1} \mbE \norm{\gr f \left( \frac{X_{s} \bone_n}{n} \right)  \bone^{\top}_n}_F^2 \norm{ \frac{\bone \bone^{\top}}{n}- W^{t-s-1}_{\eps} }^2 \\
    & \quad +
    12 \sum_{s = 0}^{t-1} \left(3 K^2  \sum_{i=1}^{n} Q_{i,s} + 3 \mbE \norm{\gr f \left( \frac{X_{s} \bone_n}{n} \right)  \bone^{\top}_n}_F^2 \right) \frac{\beta_{\eps}^{t-s-1}}{ 1- \beta_{\eps}} \\
    & \quad +
    18 n \gamma_1^2 \frac{1}{(1 - \beta_{\eps})^2}.
\end{align*}

Moreover, the term $T_3$ can be bounded as follows:
\begin{align*}
    T_3 
    & =
    \mbE \norm{ \sum_{s=0}^{t-1} (Z_s - X_s) (W - W_D) \left( \frac{\bone \bone^{\top}}{n}- W^{t-s-1}_{\eps} \right) }_F^2  \\
    & \leq
    \mbE  \sum_{s=0}^{t-1} \norm{Z_s - X_s}_F^2 \norm{W - W_D}^2 \norm{ \frac{\bone \bone^{\top}}{n}- W^{t-s-1}_{\eps} }^2 \\
    & \leq
    \frac{4 n \sigma^2}{ 1 - \beta^2_{\eps} },
\end{align*}
where we used the fact that $\norm{W - W_D} \leq 2$. Now we use the bounds derived for $T_2$ and $T_3$ to bound the consensus error $M_t$ as follows:
\begin{align}
    M_t 
    & \leq
    \frac{ \al^2 \eps^2}{n} T_2 
    +
    \frac{\eps^2}{n} T_3   \nonumber \\ 
    & \leq
    \frac{2 \al^2 \eps^2 \gamma_2^2}{1 - \beta^{2}_{\eps}}   \nonumber \\ 
    & \quad +
    \frac{6 \al^2 \eps^2 K^2}{n} \sum_{s=0}^{t-1}  \sum_{i=1}^{n} Q_{i,s} \norm{ \frac{\bone \bone^{\top}}{n}- W^{t-s-1}_{\eps} }^2  \nonumber \\ 
    & \quad + 
    \frac{6 \al^2 \eps^2}{n} \sum_{s=0}^{t-1} \mbE \norm{\gr f \left( \frac{X_{s} \bone_n}{n} \right)  \bone^{\top}_n}_F^2 \norm{ \frac{\bone \bone^{\top}}{n}- W^{t-s-1}_{\eps} }^2   \nonumber \\ 
    & \quad +
    \frac{12 \al^2 \eps^2}{n} \sum_{s = 0}^{t-1} \left(3 K^2  \sum_{i=1}^{n} Q_{i,s} + 3 \mbE \norm{\gr f \left( \frac{X_{s} \bone_n}{n} \right)  \bone^{\top}_n}_F^2 \right) \frac{\beta_{\eps}^{t-s-1}}{ 1- \beta_{\eps}}   \nonumber \\ 
    & \quad +
    \frac{18 \al^2 \eps^2 \gamma_1^2}{(1 - \beta_{\eps})^2}   \nonumber \\ 
    & \quad +
    \frac{4 \eps^2 \sigma^2}{ 1 - \beta^2_{\eps} }   \nonumber \\ 
    & \leq
    \frac{2 \al^2 \eps^2 \gamma_2^2}{1 - \beta^{2}_{\eps}}
    +
    \frac{18 \al^2 \eps^2 \gamma_1^2}{(1 - \beta_{\eps})^2}
    +
    \frac{4 \eps^2 \sigma^2}{ 1 - \beta^2_{\eps} }   \nonumber \\ 
    & \quad +
    \frac{6 \al^2 \eps^2 K^2}{n} \sum_{s=0}^{t-1}  \sum_{i=1}^{n} Q_{i,s} \beta^{2(t-s-1)}_{\eps}  \nonumber \\ 
    & \quad + 
    \frac{6 \al^2 \eps^2}{n} \sum_{s=0}^{t-1} \mbE \norm{\gr f \left( \frac{X_{s} \bone_n}{n} \right)  \bone^{\top}_n}_F^2 \beta^{2(t-s-1)}_{\eps}   \nonumber \\ 
    & \quad +
    \frac{12 \al^2 \eps^2}{n} \sum_{s = 0}^{t-1} \left(3 K^2  \sum_{i=1}^{n} Q_{i,s} + 3 \mbE \norm{\gr f \left( \frac{X_{s} \bone_n}{n} \right)  \bone^{\top}_n}_F^2 \right) \frac{\beta_{\eps}^{t-s-1}}{ 1- \beta_{\eps}}   \nonumber \\ 
    & \leq
    \frac{2 \al^2 \eps^2 \gamma_2^2}{1 - \beta^{2}_{\eps}}
    +
    \frac{18 \al^2 \eps^2 \gamma_1^2}{(1 - \beta_{\eps})^2}
    +
    \frac{4 \eps^2 \sigma^2}{ 1 - \beta^2_{\eps} }   \nonumber \\ 
    & \quad +
    \frac{6 \al^2 \eps^2}{n} \sum_{s=0}^{t-1} \mbE \norm{\gr f \left( \frac{X_{s} \bone_n}{n} \right)  \bone^{\top}_n}_F^2 \left( \beta^{2(t-s-1)}_{\eps} +\frac{2 \beta_{\eps}^{t-s-1}}{ 1- \beta_{\eps}} \right)  \nonumber \\ 
    & \quad +
    \frac{6 \al^2 \eps^2}{n} K^2 \sum_{s=0}^{t-1}  \sum_{i=1}^{n} Q_{i,s} \left(\frac{2 \beta_{\eps}^{t-s-1}}{ 1- \beta_{\eps}} + \beta^{2(t-s-1)}_{\eps} \right). \label{eq:M_t}
\end{align}
As we defined earlier, we have $M_s = \frac{1}{n} \sum_{i=1}^{n} Q_{i,s}$ which simplifies \eqref{eq:M_t} to the following:
\begin{align}
    M_t 
    & \leq
    \frac{2 \al^2 \eps^2 \gamma_2^2}{1 - \beta^{2}_{\eps}}
    +
    \frac{18 \al^2 \eps^2 \gamma_1^2}{(1 - \beta_{\eps})^2}
    +
    \frac{4 \eps^2 \sigma^2}{ 1 - \beta^2_{\eps} }   \nonumber \\ 
    & \quad +
    \frac{6 \al^2 \eps^2}{n} \sum_{s=0}^{t-1} \mbE \norm{\gr f \left( \frac{X_{s} \bone_n}{n} \right)  \bone^{\top}_n}_F^2 \left( \beta^{2(t-s-1)}_{\eps} +\frac{2 \beta_{\eps}^{t-s-1}}{ 1- \beta_{\eps}} \right)  \nonumber \\ 
    & \quad +
    6 \al^2 \eps^2 K^2 \sum_{s=0}^{t-1}  M_s \left(\frac{2 \beta_{\eps}^{t-s-1}}{ 1- \beta_{\eps}} + \beta^{2(t-s-1)}_{\eps} \right) \label{eq:M_t-2}
\end{align}
Now we can sum \eqref{eq:M_t-2} over $t=0,1,\cdots,T-1$ which yields
\begin{align}
    \sum_{t=0}^{T-1} M_t 
    & \leq
    \frac{2 \al^2 \eps^2 \gamma_2^2}{1 - \beta^{2}_{\eps}} T
    +
    \frac{18 \al^2 \eps^2 \gamma_1^2}{(1 - \beta_{\eps})^2} T
    +
    \frac{4 \eps^2 \sigma^2}{ 1 - \beta^2_{\eps} } T   \nonumber \\ 
    & \quad +
    \frac{6 \al^2 \eps^2}{n} \sum_{t=0}^{T-1} \sum_{s=0}^{t-1} \mbE \norm{\gr f \left( \frac{X_{s} \bone_n}{n} \right)  \bone^{\top}_n}_F^2 \left( \beta^{2(t-s-1)}_{\eps} +\frac{2 \beta_{\eps}^{t-s-1}}{ 1- \beta_{\eps}} \right)  \nonumber \\ 
    & \quad +
    6 \al^2 \eps^2 K^2 \sum_{t=0}^{T-1} \sum_{s=0}^{t-1}  M_s \left(\frac{2 \beta_{\eps}^{t-s-1}}{ 1- \beta_{\eps}} + \beta^{2(t-s-1)}_{\eps} \right)   \nonumber \\ 
    & \leq
    \frac{2 \al^2 \eps^2 \gamma_2^2}{1 - \beta^{2}_{\eps}} T
    +
    \frac{18 \al^2 \eps^2 \gamma_1^2}{(1 - \beta_{\eps})^2} T
    +
    \frac{4 \eps^2 \sigma^2}{ 1 - \beta^2_{\eps} } T   \nonumber \\ 
    & \quad +
    \frac{6 \al^2 \eps^2}{n} \sum_{t=0}^{T-1} \mbE \norm{\gr f \left( \frac{X_{s} \bone_n}{n} \right)  \bone^{\top}_n}_F^2 \left( \sum_{k=0}^{\infty} \beta^{2k}_{\eps} + \frac{2 \sum_{k=0}^{\infty} \beta^{k}_{\eps}}{ 1- \beta_{\eps}} \right)   \nonumber \\ 
    & \quad +
    6 \al^2 \eps^2 K^2 \sum_{t=0}^{T-1}  M_t \left( \frac{2 \sum_{k=0}^{\infty} \beta^{k}_{\eps}}{ 1- \beta_{\eps}} + \sum_{k=0}^{\infty} \beta^{2k}_{\eps}\right)   \nonumber \\ 
    & \leq
    \frac{2 \al^2 \eps^2 \gamma_2^2}{1 - \beta^{2}_{\eps}} T
    +
    \frac{18 \al^2 \eps^2 \gamma_1^2}{(1 - \beta_{\eps})^2} T
    +
    \frac{4 \eps^2 \sigma^2}{ 1 - \beta^2_{\eps} } T   \nonumber \\ 
    & \quad +
    \frac{18 \al^2 \eps^2}{n (1 - \beta_{\eps})^2} \sum_{t=0}^{T-1} \mbE \norm{\gr f \left( \frac{X_{s} \bone_n}{n} \right)  \bone^{\top}_n}_F^2   \nonumber \\ 
    & \quad +
    \frac{18 \al^2 \eps^2 K^2}{(1 - \beta_{\eps})^2} \sum_{t=0}^{T-1}  M_t. \label{eq:sum_M_t}
\end{align}
Note that $\norm{\gr f \left( \frac{X_{s} \bone_n}{n} \right)  \bone^{\top}_n}_F^2 = n \norm{\gr f \left( \frac{X_{s} \bone_n}{n} \right) }^2$, which simplifies \eqref{eq:sum_M_t} as follows:
\begin{align}
    \sum_{t=0}^{T-1} M_t 
    & \leq
    \frac{2 \al^2 \eps^2 \gamma_2^2}{1 - \beta^{2}_{\eps}} T
    +
    \frac{18 \al^2 \eps^2 \gamma_1^2}{(1 - \beta_{\eps})^2} T
    +
    \frac{4 \eps^2 \sigma^2}{ 1 - \beta^2_{\eps} } T   \nonumber \\ 
    & \quad +
    \frac{18 \al^2 \eps^2}{(1 - \beta_{\eps})^2} \sum_{t=0}^{T-1} \mbE \norm{\gr f \left( \frac{X_{s} \bone_n}{n} \right) }^2   \nonumber \\ 
    & \quad +
    \frac{18 \al^2 \eps^2 K^2}{(1 - \beta_{\eps})^2} \sum_{t=0}^{T-1}  M_t. \label{eq:sum_M_t}
\end{align}
Rearranging the terms implies that
\begin{align}
    \left( 1 - \frac{18 \al^2 \eps^2 K^2}{(1 - \beta_{\eps})^2}  \right) \sum_{t=0}^{T-1} M_t 
    & \leq
    \frac{2 \al^2 \eps^2 \gamma_2^2}{1 - \beta^{2}_{\eps}} T
    +
    \frac{18 \al^2 \eps^2 \gamma_1^2}{(1 - \beta_{\eps})^2} T
    +
    \frac{4 \eps^2 \sigma^2}{ 1 - \beta^2_{\eps} } T   \nonumber \\ 
    & \quad +
    \frac{18 \al^2 \eps^2}{(1 - \beta_{\eps})^2} \sum_{t=0}^{T-1} \mbE \norm{\gr f \left( \frac{X_{s} \bone_n}{n} \right) }^2. \label{eq:sum_M_t-2}
\end{align}
Now define
\begin{align}
    D_2 \coloneqq 1 - \frac{18 \al^2 \eps^2 K^2}{(1 - \beta_{\eps})^2}, \nonumber 
\end{align}
and rewrite \eqref{eq:sum_M_t-2} as
\begin{align}
    \sum_{t=0}^{T-1} M_t 
    & \leq
    \frac{2 \al^2 \eps^2 \gamma_2^2}{(1 - \beta^{2}_{\eps}) D_2} T
    +
    \frac{18 \al^2 \eps^2 \gamma_1^2}{(1 - \beta_{\eps})^2 D_2} T
    +
    \frac{4 \eps^2 \sigma^2}{ (1 - \beta^2_{\eps}) D_2} T   \nonumber \\ 
    & \quad +
    \frac{18 \al^2 \eps^2}{(1 - \beta_{\eps})^2 D_2} \sum_{t=0}^{T-1} \mbE \norm{\gr f \left( \frac{X_{s} \bone_n}{n} \right) }^2. \label{eq:sum_M_t-3}
\end{align}
Note that from definition of $T_1$ we have
\begin{align}
    T_1
    \leq
    \frac{K^2}{n} \sum_{i=1}^{n} Q_{i,t} = K^2 M_t. \nonumber 
\end{align}
Now use the above fact in the recursive equation \eqref{eq:T1} which we started with, that is
\begin{align}
    \mbE f \left( \frac{X_{t+1} \bone_n}{n} \right)
    & \leq 
    \mbE f \left( \frac{X_{t} \bone_n}{n} \right)  \nonumber \\ 
    & \quad 
    - 
    \frac{\al \eps - \al^2 \eps^2 K }{2} \mbE \norm{\frac{ \parf(X_{t}) \bone_n}{n} }^2 
    -
    \frac{\al \eps}{2} \mbE \norm{ \gr f \left( \frac{X_{t} \bone_n}{n} \right)}^2   \nonumber \\ 
    & \quad + 
    \frac{\eps^2  K }{2n} \sigma^2
    +
    \frac{\al^2 \eps^2 K }{2 n} \gamma_2^2   \nonumber \\ 
    & \quad + 
    \frac{\al \eps K^2}{2}  M_t. \label{eq:recursive-2}
\end{align}
If we sum \eqref{eq:recursive-2} over $t=0,1,\cdots,T-1$, we get
\begin{align}
    & \quad \frac{\al \eps - \al^2 \eps^2 K }{2} \sum_{t=0}^{T-1} \mbE \norm{\frac{ \parf(X_{t}) \bone_n}{n} }^2 
    +
    \frac{\al \eps}{2} \sum_{t=0}^{T-1} \mbE \norm{ \gr f \left( \frac{X_{t} \bone_n}{n} \right)}^2   \nonumber \\ 
    & \leq
    f(0) - f^*
    +
    \frac{\eps^2  K }{2n} \sigma^2 T 
    +
    \frac{\al^2 \eps^2 K }{2 n} \gamma_2^2 T   \nonumber \\ 
    & \quad +
    \frac{\al \eps K^2}{2} \sum_{t=0}^{T-1}  M_t   \nonumber \\ 
    & \overset{\text{from \eqref{eq:sum_M_t-3}}}{\leq}
    f(0) - f^*
    +
    \frac{\eps^2  K }{2n} \sigma^2 T 
    +
    \frac{\al^2 \eps^2 K }{2 n} \gamma_2^2 T   \nonumber \\ 
    & \quad +
    \frac{\al \eps K^2}{2}  
    \left\{ 
    \frac{2 \al^2 \eps^2 \gamma_2^2}{(1 - \beta^{2}_{\eps}) D_2} T
    +
    \frac{18 \al^2 \eps^2 \gamma_1^2}{(1 - \beta_{\eps})^2 D_2} T
    +
    \frac{4 \eps^2 \sigma^2}{ (1 - \beta^2_{\eps}) D_2} T 
    \right\}   \nonumber \\ 
    & \quad +
    \frac{9 \al^3 \eps^3 K^2}{(1 - \beta_{\eps})^2 D_2} \sum_{t=0}^{T-1} \mbE \norm{\gr f \left( \frac{X_{s} \bone_n}{n} \right) }^2. \label{eq:sum_recursive}
\end{align}
We ca rearrange the terms in \eqref{eq:sum_recursive} and rewrite it as
\begin{align}
    & \quad \frac{\al \eps - \al^2 \eps^2 K }{2} \sum_{t=0}^{T-1} \mbE \norm{\frac{ \parf(X_{t}) \bone_n}{n} }^2 
    +
    \al \eps\left( \frac{1}{2} - \frac{9 \al^2 \eps^2 K^2}{(1 - \beta_{\eps})^2 D_2} \right) \sum_{t=0}^{T-1} \mbE \norm{ \gr f \left( \frac{X_{t} \bone_n}{n} \right)}^2   \nonumber \\ 
    & \leq
    f(0) - f^*
    +
    \frac{\eps^2  K }{2n} \sigma^2 T 
    +
    \frac{\al^2 \eps^2 K }{2 n} \gamma_2^2 T   \nonumber \\ 
    & \quad +
    \frac{\al \eps K^2}{2}  
    \left\{ 
    \frac{2 \al^2 \eps^2 \gamma_2^2}{(1 - \beta^{2}_{\eps}) D_2} T
    +
    \frac{18 \al^2 \eps^2 \gamma_1^2}{(1 - \beta_{\eps})^2 D_2} T
    +
    \frac{4 \eps^2 \sigma^2}{ (1 - \beta^2_{\eps}) D_2} T 
    \right\} \label{eq:sum_recursive-2}
\end{align}
Now, we define $D_1$ as follows 
\begin{align}
    D_1 \coloneqq \frac{1}{2} - \frac{9 \al^2 \eps^2 K^2}{(1 - \beta_{\eps})^2 D_2}, \nonumber 
\end{align}
and replace in \eqref{eq:sum_recursive-2} which yields
\begin{align}
    & \quad 
    \frac{1}{\al \eps T} \left\{ \frac{\al \eps - \al^2 \eps^2 K }{2} \sum_{t=0}^{T-1} \mbE \norm{\frac{ \parf(X_{t}) \bone_n}{n} }^2 
    +
    \al \eps D_1 \sum_{t=0}^{T-1} \mbE \norm{ \gr f \left( \frac{X_{t} \bone_n}{n} \right)}^2 \right\}   \nonumber \\ 
    & \leq
    \frac{1}{\al \eps T} (f(0) - f^*)
    +
    \frac{\eps}{\al} \frac{K \sigma^2}{2n} 
    +
    \al \eps \frac{ K \gamma_2^2}{2 n}    \nonumber \\ 
    & \quad +
    \frac{\al^2 \eps^2}{1 - \beta^{2}_{\eps}} \frac{K^2 \gamma_2^2}{D_2}
    +
    \frac{\al^2 \eps^2}{(1 - \beta_{\eps})^2} \frac{9 K^2 \gamma_1^2}{D_2}
    +
    \frac{\eps^2}{1 - \beta^2_{\eps}} \frac{2 K^2 \sigma^2}{D_2}. \label{eq:sum_recursive-3}
\end{align}
To balance the terms in RHS of \eqref{eq:sum_recursive-3}, we need to know how $\beta_{\eps}$ behaves with $\eps$. As we defined before, $W_{\eps} = (1 - \eps)I + \eps W$. Hence, $\lambda_i(W_{\eps}) = 1 - \eps + \eps \lambda_i(W)$. Therefore, for $\eps \leq \frac{1}{1 - \lambda_n(W)}$, we have
\begin{align}
    \beta_{\eps}
    & =
    \max \left\{|\lambda_2(W_{\eps})|,|\lambda_n(W_{\eps})| \right\}   \nonumber \\ 
    & =
    \max \left\{|1 - \eps + \eps \lambda_2(W)|,|1 - \eps + \eps \lambda_n(W)| \right\}   \nonumber \\ 
    & =
    \max \left\{1 - \eps + \eps \lambda_2(W) , 1 - \eps + \eps \lambda_n(W) \right\}   \nonumber \\ 
    & =
    1 - \eps \left( 1 - \lambda_2(W) \right). \nonumber 
\end{align}
Therefore,
\begin{align*}
    1 - \beta_{\eps}
    & = 
    \eps \left( 1 - \lambda_2(W) \right) 
    \geq \eps (1 - \beta) \\
    1 - \beta^2_{\eps}
    & =
    2 \eps \left( 1 - \lambda_2(W) \right) - \eps^2 \left( 1 - \lambda_2(W) \right)^2 
    \geq \eps (1 - \beta^2).
\end{align*}
Moreover, if $\al \eps \leq \frac{1}{K}$ we have from \eqref{eq:sum_recursive-3} that
\begin{align}
    \frac{D_1}{T} \sum_{t=0}^{T-1} \mbE \norm{ \gr f \left( \frac{X_{t} \bone_n}{n} \right)}^2 
    & \leq
    \frac{1}{\al \eps T} (f(0) - f^*)
    +
    \frac{\eps}{\al} \frac{K \sigma^2}{2n} 
    +
    \al \eps \frac{ K \gamma_2^2}{2 n}    \nonumber \\ 
    & \quad +
    \frac{\al^2 \eps}{1 - \beta^{2}} \frac{K^2 \gamma_2^2}{D_2}
    +
    \frac{\al^2}{(1 - \beta)^2} \frac{9 K^2 \gamma_1^2}{D_2}
    +
    \frac{\eps}{1 - \beta^2} \frac{2 K^2 \sigma^2}{D_2}. \label{eq:sum_recursive-4}
\end{align}
For $\al \leq \frac{1 - \beta}{6 K}$ we have
\begin{align}
    D_2 
    & =
    1 - \frac{18 \al^2 \eps^2 K^2}{(1 - \beta_{\eps})^2}   \nonumber \\ 
    & =
    1 - \frac{18 \al^2 \eps^2 K^2}{ \eps^2 (1 - \beta)^2}   \nonumber \\ 
    & =
    1 - \frac{18 \al^2 K^2}{(1 - \beta)^2}   \nonumber \\ 
    & \geq 
    \frac{1}{2},
\end{align}
and for $\al \leq \frac{1 - \beta}{6 \sqrt{2} K }$ we have
\begin{align*}
    D_1 
    & =
    \frac{1}{2} - \frac{9 \al^2 \eps^2 K^2}{(1 - \beta_{\eps})^2 D_2} \\
    & \geq
    \frac{1}{2} - \frac{18 \al^2 \eps^2 K^2}{ \eps^2 (1 - \beta)^2} \\
    & =
    \frac{1}{2} - \frac{18 \al^2 K^2}{(1 - \beta)^2} \\
    & \geq 
    \frac{1}{4}.
\end{align*}
Now, we pick the step-sizes as follows:
\begin{align} \label{eq:al-eps}
    \al & = \frac{1}{T^{1/6}}, \\
    \eps & = \frac{1}{T^{1/2}}.
\end{align}
It is clear that in order to satisfy the conditions mentioned before, that are $\eps \leq \frac{1}{1 - \lambda_n(W)}$, $\al \eps \leq \frac{1}{K}$ and $\al \leq \frac{1 - \beta}{6 \sqrt{2} K }$, it suffices to pick $T$ as large as the following:
\begin{equation} \label{eq:minT}
    T \geq T^{\mathsf{nc}}_{\mathsf{min}} \coloneqq \max \left\{ \left( 1 - \lambda_n(W) \right)^2,K^{3/2}, \left( \frac{6 \sqrt{2} K}{1 - \beta} \right)^6 \right\}.
\end{equation}
For such $T$ we have
\begin{align}
    \frac{1}{T} \sum_{t=0}^{T-1} \mbE \norm{ \gr f \left( \frac{X_{t} \bone_n}{n} \right)}^2 
    & \leq
    \frac{1}{T^{1/3}} 4 (f(0) - f^*)
    +
    \frac{1}{T^{1/3}} \frac{2 K \sigma^2}{n} 
    +
    \frac{1}{T^{2/3}} \frac{2 K \gamma_2^2}{ n}  \nonumber\\
    & \quad +
    \frac{1}{T^{5/6}} \frac{8 K^2 \gamma_2^2}{1 - \beta^{2}}
    +
    \frac{1}{T^{1/3}} \frac{72 K^2 \gamma_1^2}{(1 - \beta)^2}
    +
    \frac{1}{T^{1/2}} \frac{16 K^2 \sigma^2}{1 - \beta^2} \nonumber \\
    & =
    \frac{B_1}{T^{1/3}} + \frac{B_2}{T^{1/2}} + \frac{B_3}{T^{2/3}} + \frac{B_4}{T^{5/6}} \label{eq:convergence} \\
    & =
    \ccalO \left( \frac{ K^2 }{(1 - \beta)^2} \frac{\gamma^2}{m} +  K \frac{ \sigma^2}{n} \right) \frac{1}{T^{1/3}}  \nonumber \\
    & \quad +
    \ccalO \left( \frac{K^2}{1 - \beta^2} \sigma^2 \right) \frac{1}{T^{1/2}}  \nonumber \\
    & \quad +
    \ccalO \left( K \frac{\gamma^2}{n} \max{\left\{\frac{\mbE[1/V]}{T_d},  \frac{1}{m}\right\}} \right) \frac{1}{T^{2/3}}  \nonumber \\
    & \quad +
    \ccalO \left( \frac{K^2}{1 - \beta^{2}} \gamma^2 \max{\left\{\frac{\mbE[1/V]}{T_d},  \frac{1}{m}\right\}} \right) \frac{1}{T^{5/6}}, \nonumber 
\end{align}
where
\begin{align*}
    B_1 & \coloneqq 4(f(0) - f^*) + \frac{72 K^2 \gamma_1^2}{(1 - \beta)^2} +  \frac{2 K \sigma^2}{n} \\
    B_2 & \coloneqq  \frac{16 K^2 \sigma^2}{1 - \beta^2} \\
    B_3 & \coloneqq \frac{2 K \gamma_2^2}{ n}  \\
    B_4 & \coloneqq \frac{8 K^2 \gamma_2^2}{1 - \beta^{2}}.
\end{align*}
Now we bound the consensus error. From \eqref{eq:sum_M_t-3} we have
\begin{align*}
    \frac{1}{T} \sum_{t=0}^{T-1} \frac{1}{n} \sum_{i=1}^{n} \mbE \norm{ \frac{X_{t} \bone_n}{n} - \bx_{i,t} }^2
    & =
    \frac{1}{T} \sum_{t=0}^{T-1} M_t \\
    & \leq
    \frac{2 \al^2 \eps^2 \gamma_2^2}{(1 - \beta^{2}_{\eps}) D_2}
    +
    \frac{18 \al^2 \eps^2 \gamma_1^2}{(1 - \beta_{\eps})^2 D_2} 
    +
    \frac{4 \eps^2 \sigma^2}{ (1 - \beta^2_{\eps}) D_2} \\
    & \quad +
    \frac{18 \al^2 \eps^2}{(1 - \beta_{\eps})^2 D_2} \frac{1}{T} \sum_{t=0}^{T-1} \mbE \norm{\gr f \left( \frac{X_{s} \bone_n}{n} \right) }^2\\
    & \leq
    \al^2 \eps \frac{2 \gamma_2^2}{ (1 - \beta^2) D_2} 
    +
    \al^2 \frac{18 \gamma_1^2}{(1 - \beta)^2 D_2} \\
    & \quad + 
    \eps \frac{4 \sigma^2}{ (1 - \beta^2) D_2} \\
    & \quad +
    \al^2 \frac{18}{(1 - \beta)^2 D_2} \frac{1}{T} \sum_{t=0}^{T-1} \mbE \norm{\gr f \left( \frac{X_{t} \bone_n}{n} \right) }^2
\end{align*}
For the same step-sizes $\al$ and $\eps$ defined in (\ref{eq:al-eps}) and large enough $T$ as in (\ref{eq:minT}), we can use the convergence result in (\ref{eq:convergence}) which yields
\begin{align*}
    \frac{1}{T} \sum_{t=0}^{T-1} \frac{1}{n} \sum_{i=1}^{n} \mbE \norm{ \frac{X_{t} \bone_n}{n} - \bx_{i,t} }^2
    & \leq
    \frac{1}{T^{5/6}} \frac{4 \gamma_2^2}{1 - \beta^2} 
    +
    \frac{1}{T^{1/3}} \frac{36 \gamma_1^2}{(1 - \beta)^2} 
    +
    \frac{1}{T^{1/2}} \frac{8 \sigma^2}{1 - \beta^2} \\
    & \quad +
    \frac{1}{T^{1/3}} \frac{36}{(1 - \beta)^2}  \left( \frac{B_1}{T^{1/3}}
    +
    \frac{B_2}{T^{1/2}} + \frac{B_3}{T^{2/3}} + \frac{B_4}{T^{5/6}} \right) \\
    & =
    \frac{C_1}{T^{1/3}} + \frac{C_2}{T^{1/2}} + \frac{C_3}{T^{2/3}} + \frac{C_4}{T^{5/6}} + \frac{C_5}{T} + \frac{C_6}{T^{7/6}}\\
    & =
    \ccalO \left( \frac{ \gamma^2}{m(1 - \beta)^2} \right) \frac{1}{T^{1/3}} \\
    & \quad +
    \ccalO \left( \frac{\sigma^2}{1 - \beta^2} \right) \frac{1}{T^{1/2}} \\
    & \quad +
    \ccalO \left( \frac{ K^2 }{(1 - \beta)^4} \frac{\gamma^2}{m} +  \frac{K}{(1 - \beta)^2} \frac{ \sigma^2}{n} \right) \frac{1}{T^{2/3}} \\
    & \quad +
    \ccalO \left( \frac{\gamma^2}{1 - \beta^{2}}  \max{\left\{\frac{\mbE[1/V]}{T_d},  \frac{1}{m}\right\}} + \frac{ K^2 \sigma^2}{(1 - \beta)^4}\right) \frac{1}{T^{5/6}} \\
    & \quad +
    \ccalO \left( \frac{K}{(1 - \beta)^2} \frac{\gamma^2}{n} \max{\left\{\frac{\mbE[1/V]}{T_d},  \frac{1}{m}\right\}} \right) \frac{1}{T} \\
    & \quad +
    \ccalO \left( \frac{ K^2 }{(1 - \beta)^4} \gamma^2  \max{\left\{\frac{\mbE[1/V]}{T_d},  \frac{1}{m}\right\}} \right) \frac{1}{T^{7/6}},
\end{align*}
where
\begin{align*}
    C_1 & \coloneqq \frac{36 \gamma_1^2}{(1 - \beta)^2} \\
    C_2 & \coloneqq \frac{8 \sigma^2}{1 - \beta^2}  \\
    C_3 & \coloneqq \frac{36}{(1 - \beta)^2} B_1 \\
    C_4 & \coloneqq \frac{4 \gamma_2^2}{1 - \beta^2}  + \frac{36}{(1 - \beta)^2} B_2 \\
    C_5 & \coloneqq \frac{36 }{(1 - \beta)^2} B_3 \\
    C_6 & \coloneqq \frac{36 }{(1 - \beta)^2} B_4.
\end{align*}


\end{document}